\documentclass[nohyperref]{article}

\usepackage{microtype}
\usepackage{graphicx}
\usepackage{subcaption, caption}
\usepackage{booktabs} 
\usepackage{nicefrac}
\usepackage{lipsum}

\usepackage{hyperref}

\usepackage[accepted]{icml2022}

\usepackage{amsmath}
\usepackage{amssymb}
\usepackage{mathtools}
\usepackage{amsthm}
\usepackage{siunitx}

\usepackage[capitalize,noabbrev]{cleveref}

\theoremstyle{plain}
\newtheorem{theorem}{Theorem}[section]

\newtheorem{lemma}[theorem]{Lemma}

\theoremstyle{definition}

\newtheorem{assumption}[theorem]{Assumption}
\theoremstyle{remark}

\usepackage{bm}

\def\Ee{\mathbb{E}}
\def\Oe{\mathbb{O}}
\def\Re{\mathbb{R}}
\def\Se{\mathbb{S}}
\def\nN{\mathcal{N}}
\def\oO{\mathcal{O}}
\def\II{\mathbf{I}}
\def\QQ{\mathbf{Q}}
\def\UU{\mathbf{U}}
\def\VV{\mathbf{V}}

\def\trans{^{\intercal}}
\def\th{^{\textnormal{th}}}
\DeclareMathOperator{\var}{Var}
\DeclareMathOperator{\trace}{tr}
\DeclareMathOperator{\diag}{diag}
\newcommand{\cmnt}[1]{}

\newcommand\mypara[1]{\vspace{0mm}\noindent\textbf{#1}}

\usepackage[textsize=tiny]{todonotes}

\allowdisplaybreaks

\icmltitlerunning{Generalizing Gaussian Smoothing for Random Search}

\begin{document}

\twocolumn[
\icmltitle{Generalizing Gaussian Smoothing for Random Search}



\icmlsetsymbol{equal}{*}

\begin{icmlauthorlist}
\icmlauthor{Katelyn Gao}{yyy}
\icmlauthor{Ozan Sener}{zzz}
\end{icmlauthorlist}

\icmlaffiliation{yyy}{Intel Labs, Santa Clara, CA, USA}
\icmlaffiliation{zzz}{Intel Labs, Munich, Germany}

\icmlcorrespondingauthor{Katelyn Gao}{katelyn.gao@intel.com}

\icmlkeywords{random search, gradient estimation, Gaussian smoothing}

\vskip 0.3in
]



\printAffiliationsAndNotice{}  

\begin{abstract}
Gaussian smoothing (GS) is a derivative-free optimization (DFO) algorithm that estimates the gradient of an objective using perturbations of the current parameters sampled from a standard normal distribution.
We generalize it to sampling perturbations from a larger family of distributions.
Based on an analysis of DFO for non-convex functions, we propose to choose a distribution for perturbations that minimizes the mean squared error (MSE) of the gradient estimate. 
We derive three such distributions with provably smaller MSE than Gaussian smoothing. 
We conduct evaluations of the three sampling distributions on linear regression, reinforcement learning, and DFO benchmarks in order to validate our claims. 
Our proposal improves on GS with the same computational complexity, and are usually competitive with and often outperform Guided ES \citep{guidedes} and Orthogonal ES \citep{choromanski2018structured}, two computationally more expensive algorithms that adapt the covariance matrix of normally distributed perturbations.
\end{abstract}

\section{Introduction}\label{sec:intro}
In many practical applications, machine learning is complicated by the lack of analytical gradients of the objective with respect to the parameters of the predictor.
For example, a search and rescue robot could have complex mechanics that may be impossible to accurately model even with full knowledge of the terrain, and without a model of the system dynamics, the analytical gradients of the success rate with respect to the policy parameters would not be available.
On the other hand, noisy evaluations of the objective, such as Booleans indicating success, are inexpensive to obtain. 
The problem of optimizing a function with only zeroth-order evaluations is called derivative-free optimization (DFO).

Gaussian smoothing (GS) \cite{matyas1965random, nesterov2017random} is a DFO algorithm that estimates the gradient using evaluations at perturbations of the parameters, randomly sampled from the standard normal distribution and computing finite differences.
Current extensions of GS add post-processing.
\citet{polyak} and \citet{flaxman2005online} normalize the perturbations; Orthogonal ES \cite{choromanski2018structured} orthogonalizes them.
Guided ES \cite{guidedes} rotates the perturbations to be better aligned with recent gradient estimates; LMRS \cite{lmrs} rotates them to a learned subspace.
The last three approaches increase the computational complexity as they require the Gram-Schmidt process.

Although Gaussian smoothing is shown to be effective, the choice of standard normal distribution is rather arbitrary. In this paper, we generalize GS to sample perturbations from arbitrary distributions. 
We can choose distributions that optimize any desired property thanks to the proposed generalization. 
Specifically, we show that a convergence bound for stochastic gradient descent for smooth non-convex functions is proportional to the mean squared error (MSE) of the gradients.
Therefore, we select a distribution with reduced MSE of the gradient estimate, computing the MSE under the assumption that the entries of the perturbations are IID with mean zero.

Our first algorithm to reduce the MSE is \emph{Bernoulli Smoothing} (BeS), which replaces the standard normal distribution with a standardized Bernoulli distribution with probability $0.5$. After fixing the distributional family of perturbations to be Gaussian or Bernoulli, we obtain distributions that approximately minimize the MSE.
The distributions have simple analytical forms and do not require any information about the objective.
In fact, they are scaled versions of GS and BeS, with smaller variance, and so we call the resulting algorithms \emph{GS-shrinkage} and \emph{BeS-shrinkage}.
Due to the IID assumption of the entries of the perturbations, BeS, GS-shrinkage, and BeS-shrinkage have the same computational complexity as GS. 

To validate the theory, we empirically evaluate our proposed methods for DFO on linear regression since the analytical gradients are available to compute various statistics.
Results confirm that GS-shrinkage and BeS-shrinkage lead to gradient estimates with smaller MSE.
We further conduct an empirical evaluation on high-dimensional reinforcement learning (RL) benchmarks, based on locomotion or manipulation, with various budgets for trajectory simulation in the environment and linear policies.
Generally, BeS learns a superior policy to GS. 
For the locomotion environments, GS-shrinkage and BeS-shrinkage usually outperform BeS; which one is better depends on the environment and the budget.
Lastly, we evaluate on noisy DFO benchmarks, and observe that when the number of perturbations sampled at each iteration is smaller than the problem dimension, BeS outperforms GS at the start of optimization.
However, BeS and GS outperform GS-shrinkage and BeS-shrinkage.
Overall, our algorithms are computationally more efficient and usually competitive with and often outperform Guided ES and Orthogonal ES. 
These conclusions remain when a neural network replaces the linear policy in RL.

\section{Related work}\label{sec:related}
Derivative-free optimization is a research field that includes Bayesian optimization, genetic algorithms, and random search; see \citet{dfobook} and \citet{dfosurvey} for surveys.
We review in more detail literature most related to our work, Gaussian smoothing and evolutionary strategies.

\mypara{Gaussian smoothing}
GS \cite{matyas1965random, nesterov2017random} is a random search algorithm that estimates the gradient of an objective using its values at random perturbations of the parameters sampled from the standard normal distribution.
Many variants exist.
\citet{polyak} and \citet{flaxman2005online} normalize the perturbations, obtaining samples from the uniform distribution on the unit sphere instead of the standard normal.
More recently, methods have been proposed to improve GS by modifying the distribution of the perturbations, at the expense of greater computational complexity, which we have included as baselines in the experiments.
\citet{choromanski2018structured} orthogonalizes the perturbations, by either the Gram-Schmidt process or constructing random Hadamard-Rademacher matrices, which decreases the MSE of the gradient estimate.
\citet{guidedes} changes the covariance matrix of the perturbations to be aligned with the subspace spanned by recent gradient estimates, again requiring the Gram-Schmidt process, and \citet{lmrs} proposes a similar algorithm for the special case where the objective lies on a learned low-rank manifold.

Our methods do not increase the computational complexity compared to GS, as we still assume that the entries of the perturbations are IID.
We explicitly minimize the MSE of the gradient estimate with respect to the distribution of the perturbations.
\citet{stars} adopts a similar strategy to learn the optimal spacing for the finite difference in GS; unlike ours, their algorithms depend on characteristics of the objective.

\mypara{Evolutionary strategies}
Evolutionary strategies (ES), a class of genetic algorithm, mathematically looks similar to Gaussian smoothing but is orthogonally motivated.
ES minimizes the expected objective of a distribution over the parameter space, which is equivalent to minimizing the objective if the distribution is allowed to degenerate to a delta distribution.
More concretely, if the distribution were Gaussian optimization would be done with respect to both the mean and variance.
On the other hand, Gaussian smoothing and its relatives optimize only the mean; the variance is utilized purely to estimate the gradient.
Due to the difference in the algorithmic structure, we decided not to include ES algorithms as baselines in the experiments.
Popular ES algorithms are CMA-ES \cite{cmaes}, where the distribution is anisotrophic Gaussian, and NES \cite{wierstra2014natural}, which performs natural gradient descent for arbitrary distributions.

\mypara{GS for policy search}
Several of the previous works evaluate on reinforcement learning benchmarks.
\citet{salimans2017evolution} applies GS to MuJoCo locomotion and Atari environments with MLP policies, showing performance competitive with policy gradient algorithms.
However, it requires objective shaping, which \citet{choromanski2018structured} and subsequent works were able to remove.
For linear policies, ARS \cite{mania2018simple} showed that the MuJoCo lomotion benchmarks can be solved by GS after adding observation and reward standardization; \citet{lmrs} substantially speeds up learning on the more difficult environments.
We remark that the ARS gradient estimator is mathematically similar to GS-shrinkage.
However, it treats the variance of the perturbation distribution as a hyperparameter to be tuned through grid search, and so we do not include it as a baseline.
In contrast, we propose to set it to a value that approximately minimizes the MSE of the gradient estimate, without any knowledge of the objective or optimization needed.

\section{Preliminaries}\label{sec:background}
We first present the notation used in the paper and provide some background on Gaussian smoothing. 
Then, we show that the convergence bound for stochastic gradient descent (SGD) is proportional to the MSE of the gradient estimate, thereby providing the motivation behind the algorithms proposed in Section \ref{sec:algos}.

\subsection{Notation}\label{sec:notation}

We are interested in an unconstrained scalar minimization objective $F(\theta):\Re^d \to \Re$.
Suppose that it can only be accessed via a random evaluation $f(\theta,\xi)$ satisfying $\Ee_\xi f(\theta,\xi)=F(\theta)$.
For example, in supervised learning, the random evaluation could be the loss at a data point, and in reinforcement learning the negative of a trajectory reward.

First-order optimization methods estimate the gradient of the objective using samples from a gradient oracle $\nabla_\theta f(\theta,\xi)$.
Alternatively, as described next, the gradient of the objective may be estimated using only random evaluations, which are generally inexpensive.

\subsection{Gaussian smoothing}\label{sec:gaussiansmoothing}

Gaussian smoothing (GS) \cite{nesterov2017random} estimates the gradient of a function by generating a direction from a standard normal distribution, computing the directional derivative along the direction using function evaluations, and then multiplying the directional derivative with the direction.
The estimate can be plugged into any gradient-based optimization method, making Gaussian smoothing a widely applicable zero-order approach.

Specifically, the Gaussian smoothing gradient estimator is
\begin{align}\label{eq:gsestimator}
    \nabla_\theta F^{GS}(\theta) &= \dfrac{1}{c}F(\theta+c\epsilon)\epsilon, \quad \epsilon\sim\nN(0,\II).
\end{align}
It may be interpreted as a Monte Carlo estimate of the gradient of the following modified objective, after smoothing by a standard normal random variable:
\begin{align*}
    F_{c}(\theta)\triangleq\Ee_{\epsilon\sim\nN(0,\II)}[F(\theta+c\epsilon)], \quad c>0
\end{align*}
The gradient of the modified objective is given by
\begin{align}\label{eq:gsrationale}
    \nabla_\theta F_{c}(\theta) &= \Ee_{\epsilon\sim\nN(0,\II)}\left[\dfrac{1}{c}F(\theta+c\epsilon)\epsilon\right]
\end{align}
Because $\nabla_\theta F^{GS}(\theta)$ often has high variance in practice, popular alternatives are the forward-difference (FD) estimator and the antithetic (AT) estimator, which incorporate control variates: for $\epsilon\sim\nN(0,\II)$,
\begin{align}\label{eq:fdantithetic}
\begin{split}
    \nabla_\theta F^{FD}(\theta) &= \dfrac{1}{c}\left[F(\theta+c\epsilon)-F(\theta)\right]\epsilon \\
    \nabla_\theta F^{AT}(\theta) &= \dfrac{1}{2c}\left[F(\theta+c\epsilon)-F(\theta-c\epsilon)\right]\epsilon
\end{split}
\end{align}
The variance of each estimator can be reduced by averaging over multiple directions.
We focus on the FD estimator due to its simplicity and lower computational burden.

When $F(\theta)$ enjoys some mild regularity conditions and $c\to 0$, \citet{nesterov2017random} shows that iterative optimization using gradients estimated via \eqref{eq:fdantithetic} converges to a stationary point for non-convex objectives and the optimal point for convex ones.

\subsection{Convergence of biased SGD}\label{sec:sgdmse}

The estimators in \eqref{eq:fdantithetic} are unbiased for the gradient of the modified objective $F_c(\theta)$, but are biased for the gradient of the objective of interest $F(\theta)$.
Therefore, the usual convergence guarantees for iterative optimization, which generally assume unbiased gradient estimates, do not directly hold.
Recent works have analyzed the convergence properties of SGD with biased gradient estimates, for convex \cite{hu2016bandit} and smooth \cite{chen2018stochastic, ajalloeian2020analysis} objectives.
In Theorem \ref{thm:convergencemse}, we provide a convergence guarantee on SGD with biased gradient estimates that, in contrast to those works, depends on the MSE of the gradient estimates.
The proof is given in Appendix \ref{sec:sgdproof}.

\begin{theorem}\label{thm:convergencemse}
Assume that i) $F(\theta)$ is differentiable, $\mu$-smooth \footnote{$F(\theta)$ is $\mu$-smooth if $\|\nabla_\theta F(\theta_1)-\nabla_\theta F(\theta_2)\|_2\leq\mu\|\theta_1-\theta_2\|_2$ for any $\theta_1$ and $\theta_2$ in the domain of $F(\theta)$.}, and is bounded by $\Delta$ ii) the bias and the MSE of the gradient estimates $g^t$ satisfy $\|\Ee[g^t]-\nabla_\theta F(\theta^t)\|_2\leq B\|\nabla_\theta F(\theta^t)\|_2$ and $\Ee[\|g^t-\nabla_\theta F(\theta^t)\|_2^2]\leq M$, respectively iii) iterative updates are applied via $\theta^{t+1}=\theta^t-\eta g(\theta^t)$ for $T$ steps. Then, if $B< 0.5$, letting $\eta=\nicefrac{1}{\mu\sqrt{T}}$,
\begin{align*}
    \dfrac{1}{T}\sum_t\Ee\left(\|\nabla_\theta F(\theta^t)\|_2^2\right) &\leq \dfrac{M+4\Delta\mu}{(1-2B)\sqrt{T}}.
\end{align*}
\end{theorem}
This guarantee suggests that we may improve convergence by reducing the MSE of the gradient estimates.
Ideally, we would minimize the entire bound, but doing so is impractical as it requires knowledge of $\mu$ and $\Delta$.

The assumptions in Theorem \ref{thm:convergencemse} are standard in the analysis of convergence of iterative optimization for non-convex functions, but are fairly strong.
In particular, the objective is not bounded for linear regression.
However, we only use that assumption to bound the difference between the initial function value $F(\theta^0)$ and the optimal one $F(\theta^\star)$. 
Hence, it may be replaced with an assumption that the initialization has bounded distance from the optimal solution.

\section{Generalized smoothing}\label{sec:algos}
Our main idea is that by generalizing GS to be able to sample from arbitrary distributions, not just the standard normal, we may select the distribution to optimize any criterion.
In this paper, we propose to select the distribution that minimizes the MSE of the gradient estimates \eqref{eq:fdantithetic}, aiming to improve the convergence of SGD using those gradient estimates.

In this section, we present our proposed algorithms, with all proofs in Appendix \ref{sec:mainproofs}.
We focus on the forward difference gradient estimator where the objective is estimated via Monte Carlo sampling random evaluations; derivations for the antithetic gradient estimator are the same (see Appendix \ref{sec:antithetic}).
Mathematically, given $N$ random evaluations $\xi_i$ and $L$ IID sampled directions $\epsilon_l$, our estimator of interest is

\begin{align}\label{eq:fdempirical}
    \nabla_\theta \hat{F}^{FD}(\theta) &= \dfrac{1}{cLN}\sum_{l,i} (f(\theta+c\epsilon_l,\xi_i)-f(\theta,\xi_i))\epsilon_l.
\end{align}
We also make the following assumption:
\begin{assumption}\label{assume:perturb}
The entries of $\epsilon_l$, $\{\epsilon_{lj}\}_{j=1}^d$, are IID samples from a distribution with expectation $0$.
\end{assumption}

\subsection{MSE of the FD estimator}\label{sec:msefd}

We start by computing the MSE of $\nabla_\theta \hat{F}^{FD}(\theta)$. 

\begin{lemma}\label{thm:fdempmse}
Suppose that i) the first-order Taylor expansions of $F(\theta)$ and $f(\theta,\cdot)$ satisfy a regularity condition \footnote{See Appendix \ref{sec:mainproofs} for details.} ii) assumption \ref{assume:perturb} holds. Then, as $c\to 0$, the MSE of $\nabla_\theta \hat{F}^{FD}(\theta)$ approaches 
\begin{align}
    &\left((\sigma^2-1)^2+\dfrac{\sigma^4}{L}(d+k-2)\right)\|\nabla_\theta F(\theta)\|_2^2 \label{eq:msepart1} \\
    &\quad+\dfrac{\sigma^4}{LN}(d+k-1)\trace(\var_\xi[\nabla_\theta f(\theta,\xi)]), \label{eq:msepart2}
\end{align}
where $\sigma^2$ and $k$ are the variance and kurtosis of $\epsilon_{lj}$. 
The bias of $\nabla_\theta \hat{F}^{FD}(\theta)$ approaches $(\sigma^2-1)\nabla_\theta F(\theta)$.
\end{lemma}
Notice that GS makes the same assumptions as Lemma \ref{thm:fdempmse}, except for (i). For the rest of this section, we operate in the setting where $c\to 0$.

\subsection{Bernoulli Smoothing}\label{sec:brs}

There is a very simple choice of $\epsilon_{lj}$ that reduces the MSE of $\nabla_\theta \hat{F}^{FD}(\theta)$.
For GS, $\epsilon_{lj}\sim\nN(0,1)$, so $\sigma^2=1$ and $k=3$ \cite{normal}, and the MSE of $\nabla_\theta \hat{F}^{FD}(\theta)$ equals $$\dfrac{d+1}{L}\|\nabla_\theta F(\theta)\|_2^2+\dfrac{d+2}{LN}\trace(\var_\xi[\nabla_\theta f(\theta,\xi)]).$$

Observe that if $\sigma^2=1$ is fixed (and the gradient estimate remains asymptotically unbiased), the MSE of $\nabla_\theta \hat{F}^{FD}(\theta)$ decreases with smaller kurtosis $k$.
Since the Bernoulli distribution has the smallest kurtosis of any distribution \cite{DeCarlo1997OnTM}, a natural proposal to reduce the MSE is $\epsilon_{lj}\sim (B_{0.5}-0.5)/0.5$, where $B_{0.5}$ follows the Bernoulli distribution with probability $0.5$.
In other words, $\epsilon_{lj}$ is a standardized fair Bernoulli random variable with expectation $0$, $\sigma^2=1$, and $k=1$ \cite{Bernoulli}, and the MSE of $\nabla_\theta \hat{F}^{FD}(\theta)$ would equal
$$\dfrac{d-1}{L}\|\nabla_\theta F(\theta)\|_2^2+\dfrac{d}{LN}\trace(\var_\xi[\nabla_\theta f(\theta,\xi)]),$$
smaller than its MSE for GS.

We call this proposal \textit{Bernoulli Smoothing} (BeS). 
It remains unknown whether there is a direct correspondence to the gradient of a smoothed objective, like for GS.
Note that the bias $B$ remains zero, so the corresponding convergence bound in Theorem \ref{thm:convergencemse} is smaller than that of GS; experimental results in Section \ref{sec:experiments} confirm that BeS is always at least competitive with, and usually outperforms, GS.

\subsection{Shrinkage gradient estimators}\label{sec:shrinkage}

We next take a more principled approach to reduce the MSE of $\nabla_\theta \hat{F}^{FD}(\theta)$ and find the distribution of $\epsilon_{lj}$ that minimizes it.
For mathematical tractability, we restrict to a certain distribution type (Gaussian or Bernoulli).
This method can easily be extended to other distribution types.

\mypara{Gaussian}
Since a Gaussian distribution is determined by its expectation and variance, has kurtosis $3$, and we assume that $\epsilon_{lj}$ has expectation $0$, we search over the variance $\sigma^2$. In particular, we minimize only the larger term of the MSE, \eqref{eq:msepart1}, since minimizing the entire MSE requires the gradients of $F(\theta)$ and $f(\theta,\xi)$; then, the problem reduces to solving
\begin{align}\label{eq:gsminimize}
    \min_{\sigma^2>0}\quad(\sigma^2-1)^2+\dfrac{\sigma^4}{L}(d+1)
\end{align}
Doing so is reasonable for learning from mini-batch samples because \eqref{eq:msepart2} is $\oO(1/N)$ smaller than \eqref{eq:msepart1}, where $N$ is the batch size. 
Moreover, the problem is simplified as \eqref{eq:gsminimize} is quadratic in $\sigma^2$ and thus has an analytic solution.

\begin{theorem}\label{thm:gsshrinkage}
The solution to \eqref{eq:gsminimize} is $\sigma^{2*}=\nicefrac{L}{L+d+1}$. When $\epsilon_{lj}\sim\nN(0,\sigma^{2*})$, the MSE of $\nabla_\theta \hat{F}^{FD}(\theta)$ is smaller than when $\epsilon_{lj}\sim\nN(0,1)$.
\end{theorem}
Notice that the variance of $\epsilon_{lj}$ has been shrunk towards zero, and the shrinkage increases as the data dimension $d$ increases. We call setting $\epsilon_{lj}\sim\nN(0,\sigma^{2*})$ \textit{GS-shrinkage}.

\mypara{Bernoulli}
Extending BeS, we consider $\epsilon_{lj}\sim (B_p-p)/m$, where $B_p$ follows the Bernoulli distribution with probability $p$, and search over $p$ and $m$.
$\epsilon_{lj}$ is centered, with variance $p(1-p)/m^2$ and kurtosis $3+\nicefrac{1-6p(1-p)}{p(1-p)}$.
As with the Gaussian case, we minimize only \eqref{eq:msepart1}, and the problem reduces to solving
\begin{align}\label{eq:bsminimize}
\begin{split}
    &\min_{p\in(0,1),m>0}\left(\dfrac{p(1-p)}{m^2}-1\right)^2 \\
    &\quad\quad\quad\quad+\dfrac{p^2(1-p)^2}{Lm^4}\left(d+1+\dfrac{1-6p(1-p)}{p(1-p)}\right)
\end{split}
\end{align}
Because \eqref{eq:bsminimize} is quadratic in $p(1-p)$ for fixed $m$, it can be solved analytically.

\begin{theorem}\label{thm:bsshrinkage}
Assume $L+d>5$. The solution to \eqref{eq:bsminimize} is $p^*=0.5$ and $m^*=\sqrt{\nicefrac{L+d-1}{4L}}$. When $\epsilon_{lj}\sim(B_p^*-p^*)/m^*$, the MSE of $\nabla_\theta \hat{F}^{FD}(\theta)$ is smaller than when $\epsilon_{lj}\sim (B_{0.5}-0.5)/0.5$.
\end{theorem}
Notice that the variance of $\epsilon_{lj}$ has been shrunk towards zero, but the kurtosis is unchanged; again the amount of shrinkage increases as the data dimension $d$ increases. Therefore, we call setting $\epsilon_{lj}\sim(B_p^*-p^*)/m^*$ \textit{BeS-shrinkage}.

Table \ref{tab:algorithms} summarizes our three proposed algorithms and GS, and compare the corresponding MSEs of $\nabla_\theta \hat{F}^{FD}(\theta)$.
All three have the advantage that the distributions of $\epsilon_{lj}$ are objective-independent.
Assumption \ref{assume:perturb} ensures that BeS, GS-shrinkage, and BeS-shrinkage have the same computational complexity to sample the directions as GS, $\oO(Ld)$.

While the variance of $\epsilon_{lj}$ are similar for GS-shrinkage and BeS-shrinkage, the difference in the MSEs depends on $L$, $N$, $d$ and the magnitude of the gradients of $F(\theta)$ and $f(\theta,\xi)$ and may be substantial.
In particular, for high-dimensional problems, BeS-shrinkage can have smaller MSE at the beginning of optimization when $\nabla_\theta F(\theta)$ is large, while GS-shrinkage has smaller MSE at the end of optimization (see Appendix \ref{sec:gsvsbsshrinkage} for further details).
This suggests that BeS-shrinkage could lead to better sample efficiency for online reinforcement learning, where new trajectories are generated at each optimization iteration.

\begin{table*}[t]
\caption{Comparison of our proposed algorithms and GS for gradient estimation via \eqref{eq:fdempirical}. Recall that $\epsilon_{lj}$ is the random variable that each entry of the direction is sampled from, $L$ is the number of sampled directions, $d$ is the dimension of the parameter space, and $B_p$ is a Bernoulli random variable with probability $p$.}
\label{tab:algorithms}
\vskip 0.1in
\begin{center}
\begin{small}
\begin{sc}
\begin{tabular}{llcr}
\toprule
Algorithm & Distribution of $\epsilon_{lj}$ & Smaller MSE than \\
\midrule
GS & $\nN(0,1)$ & \\
BeS & $(B_{0.5}-0.5)/0.5$ & GS\\
GS-shrinkage & $\nN(0,\nicefrac{L}{L+d+1})$ & GS \\
BeS-shrinkage & $(B_{0.5}-0.5)/m^*$, $m^*=\sqrt{\nicefrac{L+d-1}{4L}}$ & BeS \\
\bottomrule
\end{tabular}
\end{sc}
\end{small}
\end{center}
\vskip -0.1in
\end{table*}

\section{Experiments}\label{sec:experiments}
We conduct experiments to i) validate the theoretical claims in Sections \ref{sec:background} and \ref{sec:algos} ii) compare the three proposed algorithms to GS and two previous algorithms, guided ES \cite{guidedes} and orthogonal ES \cite{choromanski2018structured}, and iii) investigate the impact of increasing the dimension $d$ or of using the antithetic gradient estimator instead of the forward difference gradient estimator.
The optimizer is SGD using the gradient estimator \eqref{eq:fdempirical} (or in (iii), its antithetic version).
The learning rate and spacing $c$ are chosen by grid search, to maximize the test performance at the end of optimization \footnote{The learning rate suggested by Theorem \ref{thm:convergencemse} is not practical to compute since it includes the smoothness of the objective.}.
Details are found in Appendix \ref{sec:expdetails} and code is available at \url{https://github.com/isl-org/generalized-smoothing}.

\subsection{Validating theory}\label{sec:validation}

\begin{figure*}[ht]
\includegraphics[width=\textwidth]{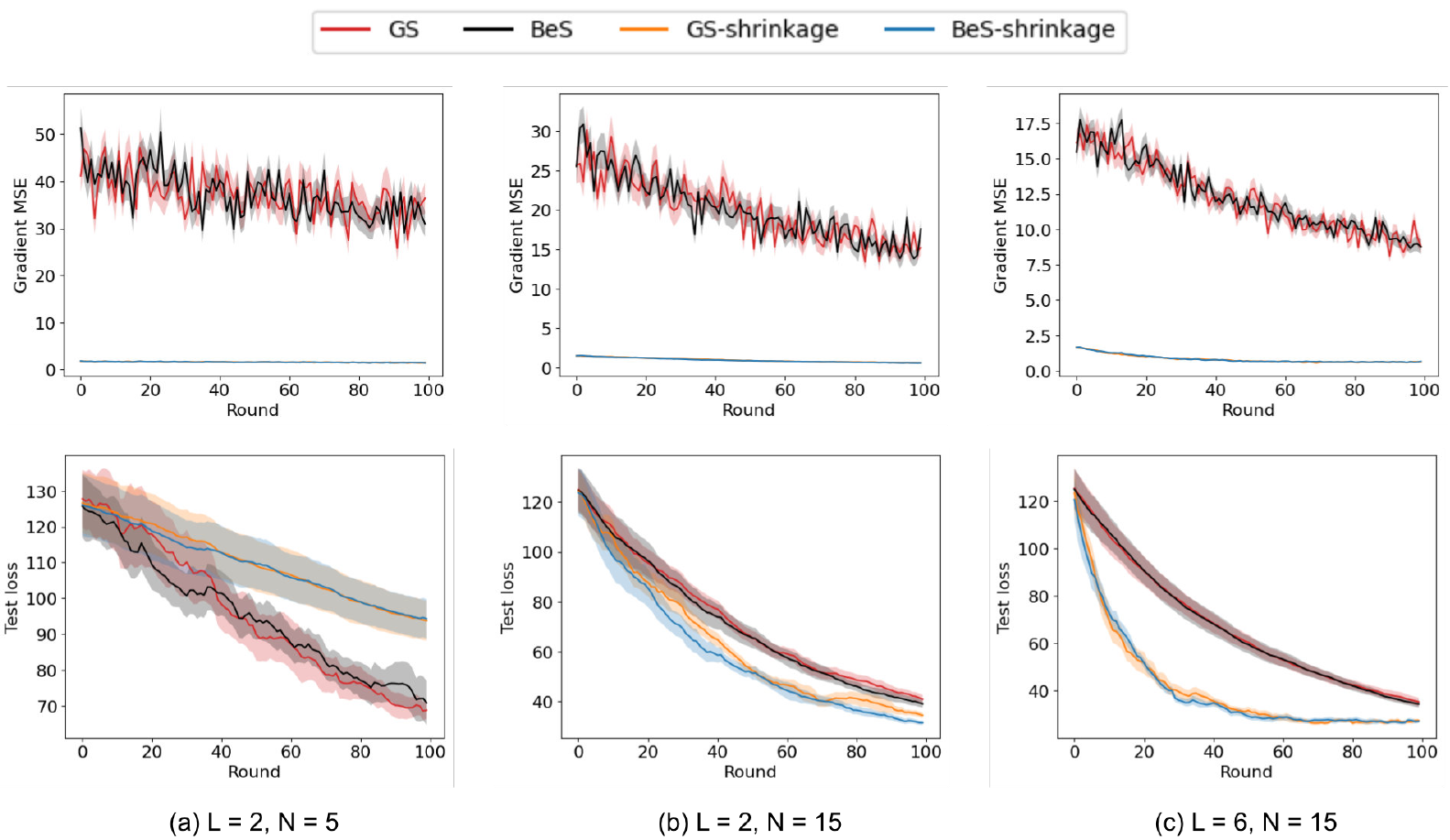}
\vskip -0.1in
\caption{For linear regression with various $L$ and $N$: MSE of the gradient (row $1$) and test loss (row $2$).}
\vskip -0.1in
\label{fig:linreg}
\end{figure*}

To validate the theoretical claims, we evaluate GS, BeS, GS-shrinkage, and BeS-shrinkage on linear regression with squared error loss, where analytical formulas for the gradient of the objective are available.
Our data model is from \citet{gao2020modeling}:
\begin{align}\label{eq:linregmodel}
\begin{split}
    &y=\gamma\trans x+\epsilon \quad \epsilon\sim\nN(0,\sigma^2) \quad x\sim\nN(0,\QQ) \\
    &\gamma\sim U([0,2]^d) \quad \sigma^2\sim U([0,2]) \\
    &\QQ=\VV\diag(\gamma)\VV\trans \quad \VV\sim U(\Se\Oe(d))
\end{split}
\end{align}

We show results for $d=100$ in the online setting, where $N$ new data points are sampled from the model at each optimization iteration.
The first row of Figure \ref{fig:linreg} plots the MSEs of the gradient estimator \eqref{eq:fdempirical} as optimization progresses, for the four algorithms over a range of values for $L$ and $N$.
The second row plots the corresponding losses on a test set, generated from \eqref{eq:linregmodel}.
Appendix \ref{sec:linregdetails} contains similar plots for more values of $L$ and $N$.
We control for $L$ and $N$ since they directly affect the MSE, as seen from Lemma \ref{thm:fdempmse}; the average is taken over five randomly generated seeds and the bands indicate one standard deviation.

The MSE of the gradient for GS-shrinkage and BeS-shrinkage is always substantially smaller than for GS and BeS; the differences between GS and BeS and between GS-shrinkage and BeS-shrinkage are not statistically significant.
In terms of the test loss, the story is mixed, but the main difference is between GS/BeS and GS-shrinkage/BeS-shrinkage.
For $L=2, N=5$, standard errors are large and GS \& BeS appear to outperform GS-shrinkage \& BeS-shrinkage.
However, for $L=2$ \& $N=15$ and $L=6$ \& $N=15$, GS-shrinkage \& BeS-shrinkage statistically significantly outperform GS \& BeS; BeS-shrinkage appears to be slightly better in the first case, while the two are competitive in the second case.
This suggests that the lower MSE of BeS-shrinkage compared to GS-shrinkage at the beginning of optimization may indeed translate to better test performance.

\subsection{Online reinforcement learning}\label{sec:comparisons}

\begin{figure*}[h!]
\vskip 0.1in
\includegraphics[width=\textwidth]{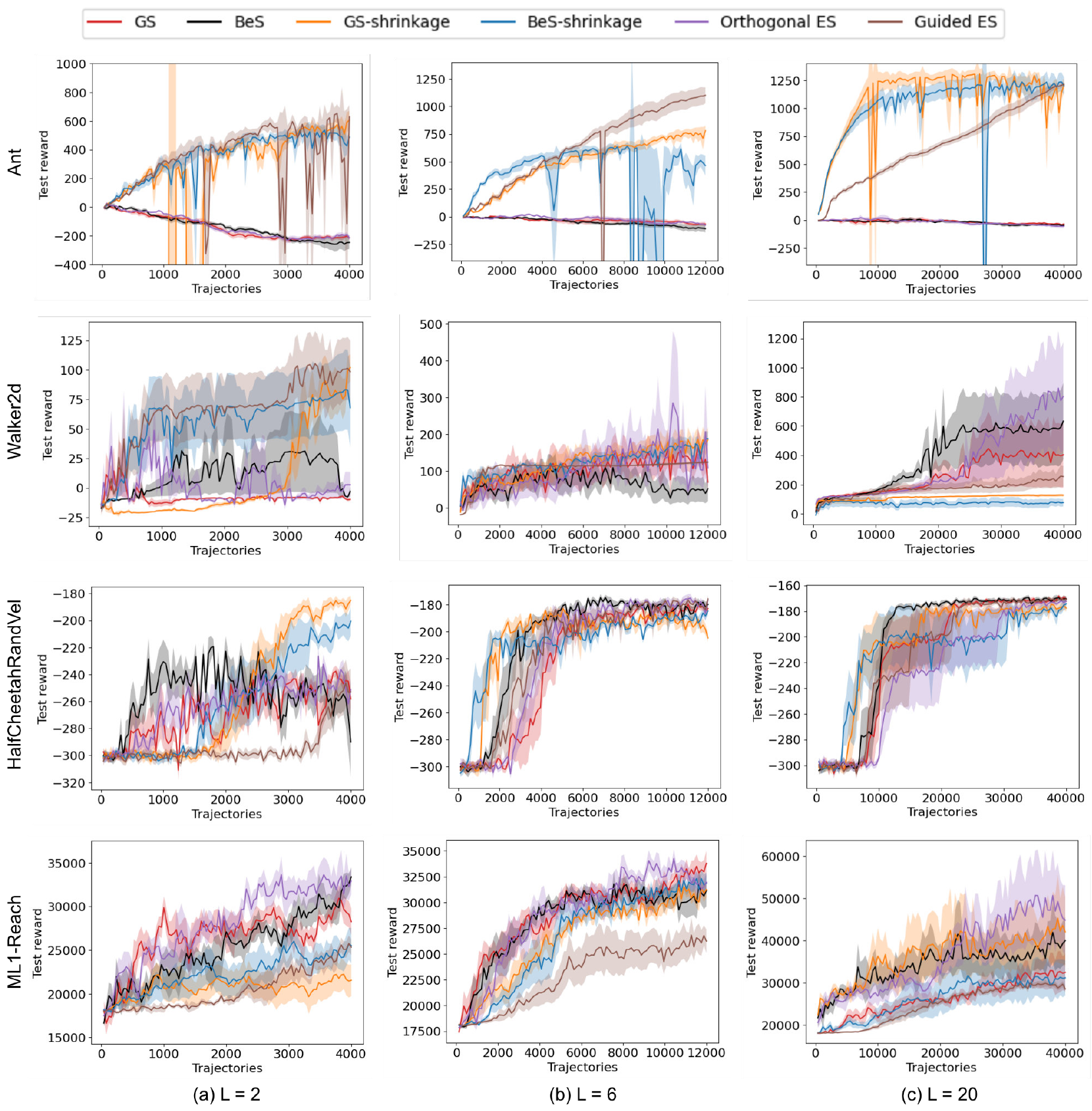}
\caption{For RL with various $L$: test episodic reward}
\label{fig:rl}
\vskip -0.1in
\end{figure*}

The remaining experiments compare BeS, GS-shrinkage, and BeS-shrinkage to GS and two algorithms from the literature that also aim to improve GS by choosing the distribution from which the directions are sampled to satisfy some criterion.
In order to speed up convergence, Guided ES \cite{guidedes} samples from a Gaussian distribution whose covariance matrix incorporates previous gradient estimates during optimization.
Orthogonal ES \cite{choromanski2018structured} samples from a standard normal distribution and then orthogonalizes the directions, which reduces the MSE of the gradient estimate compared to GS.

We experiment on four RL benchmarks based on the MuoJoCo physics simulator \cite{mujoco}.
Two environments are classic locomotion environments, where the goal is to learn a policy that successfully walks: \emph{i) Ant} and \emph{ii) Walker2d} from OpenAI Gym \cite{gym}.
The other two environments are from meta-RL, where the goal is to learn a policy that succeeds over a distribution of tasks: \emph{iii) ML1-Reach}: Introduced in \citet{metaworld}, tasks correspond to moving a robot arm to random locations. \emph{iv) HalfCheetahRandVel} \cite{maml}: tasks correspond to HalfCheetah locomotion robots with random target velocities. We use the version provided in the repository for \citet{promp}.
Experiments on additional environments are in Appendix \ref{sec:rldetails}.

Concretely, the objective is the episodic reward of a linear policy.
Following \citet{mania2018simple}, during optimization we standardize the observations, divide the rewards at each iteration by their standard deviation, and remove the survival bonus of Ant and Walker2d.
Dividing the rewards at each iteration by their standard deviation stabilizes the optimization and removes the need for a tuned learning rate schedule, but may also diminish the benefit conferred by a decrease in gradient estimate MSE; therefore, GS has an advantage here.
Figure \ref{fig:rl} plots the episodic reward of the learned policy, tested on reinitializations of the environment (which includes the task for meta-RL); the horizontal axis is the number of trajectories generated in the simulator.
Since the total number of evaluations to compute the gradient estimate is $2LN$, we see from Lemma \ref{thm:fdempmse} that given a budget of evaluations that we can obtain at one time, having $L$ be as large as possible minimizes the MSE.
Thus, we show results for a range of values of $L$ and $N=1$, averaging over five randomly generated seeds; in this setting $L$ would correspond to the number of robots available to collect data in the real world.

Table \ref{tab:rltimes} displays the average computation time required to sample the directions in each optimization iteration for each algorithm on HalfCheetahRandVel, using $L$ Xeon E7-8890 v3 CPUs; these numbers only depend on the parameter dimension $d$.
We do not include the time required to generate the trajectories in the simulator, as it would be the same for all algorithms and vary between different simulators.
BeS, GS-shrinkage, and BeS-shrinkage have the same computational complexity as GS, but Guided ES and Orthogonal ES have higher complexity because they require Gram-Schmidt orthonormalization.
As expected, BeS, GS-shrinkage, and BeS-shrinkage have similar direction sampling time to GS, while Orthogonal ES takes at least $3\times$ more time and Guided ES at least $10\times$ more time.

In the majority of cases in Figure \ref{fig:rl}, BeS learns more quickly and achieves higher reward than GS.
GS-shrinkage and BeS-shrinkage usually improve on BeS in the three locomotion environments, in particular when fewer trajectories are generated at each iteration.
For Ant, GS-shrinkage and BeS-shrinkage outperform the other algorithms except Guided ES by a large margin, albeit with some instability, which may be reduced by adding momentum to the optimization.
They are competitive with Guided ES when $L=2$ and outperform it when $L=20$ and $L=6$ (at the beginning of optimization).
For Walker2d $L=2$, BeS-shrinkage and Guided ES outperform all other algorithms; when $L=6$, BeS-shrinkage, GS-shrinkage, and Orthogonal ES improve slightly on the others, and when $L=20$, BeS and Orthogonal ES are the best at the beginning and end of optimization, respectively.
For HalfCheetahRandVel, BeS-shrinkage and GS-shrinkage learn a successful policy the fastest, but the other algorithms are able to catch up by the end of optimization for $L=6$ and $L=20$.
However, for ML1-Reach, Orthogonal ES is the overall best algorithm, followed by BeS (and GS when $L=2$ and $L=6$).
Altogether, our proposed algorithms are at least competitive with Orthogonal ES and Guided ES, with the exception of ML1-Reach.

\begin{table}[t]
\caption{Time ($10^{-5}$ s) to sample directions in each optimization iteration for each algorithm, on HalfCheetahRandVel.}
\label{tab:rltimes}
\vskip 0.1in
\begin{center}
\begin{sc}
\resizebox{\columnwidth}{!}{
\begin{tabular}{@{}l@{}S[
		table-number-alignment = right,
		separate-uncertainty = true,
		table-figures-uncertainty = 1
	]@{}S[
		table-number-alignment = right,
		separate-uncertainty = true,
		table-figures-uncertainty = 1
	]@{}S[
		table-number-alignment = right,
		separate-uncertainty = true,
		table-figures-uncertainty = 1
	]@{}}
\toprule
Algorithm & \multicolumn{1}{r}{$L=2$} & \multicolumn{1}{r}{$L=6$} & \multicolumn{1}{r}{$L=20$} \\
\midrule
GeS            & 7.5 \pm 0.1  & 13.1 \pm 0.1 & 19.9 \pm 0.1 \\ 
BeS            & 8.9 \pm 0.1  & 13.5 \pm 0.1  & 21.6 \pm 0.2 \\ 
GeS-shrinkage  & 8.5 \pm 0.1  & 13.6 \pm 0.1 & 24.0 \pm 0.2 \\ 
BeS-shrinkage  & 8.7 \pm 0.1  & 14.2 \pm 0.1 & 19.1 \pm 0.1 \\ 
Orthogonal ES & 22.1 \pm 0.1 & 50.8 \pm 0.2 & 96.4 \pm 0.5 \\
Guided ES     & 193 \pm 1.9  & 228 \pm 1.8  & 222 \pm 2.2 \\ 
\bottomrule
\end{tabular}}
\end{sc}
\end{center}
\vskip -0.1in
\end{table}

\subsection{Ablations}\label{sec:ablations}

Our next experiment studies whether the conclusions in the previous subsection generalize to using i) a neural network policy or ii) the antithetic gradient estimator.
We repeat the set-up of Section \ref{sec:comparisons}, but only on Ant with $L=20$.
Figure \ref{fig:rlablations} plots the test episodic reward of the learned policy against the number of trajectories generated during optimization.

\mypara{Neural network policy}
Instead of a linear policy, we use a MLP policy with one hidden layer of $32$ nodes and tanh activations, which is popular in the policy gradient RL literature \cite{maml}.
GS-shrinkage and BeS-shrinkage outperform the other algorithms, with BeS-shrinkage learning statistically significantly faster.
Guided ES falls behind GS, BeS, and Orthogonal ES, which start out strong but deteriorates, indicating difficulty with tuning learning rates.

\mypara{Antithetic gradient estimator}
Instead of the forward difference gradient estimator \eqref{eq:fdempirical}, we use the antithetic gradient estimator \eqref{eq:atempirical}.
We show in Appendix \ref{sec:antithetic} that doing so does not affect the validity of BeS, GS-shrinkage, and BeS-shrinkage, but may be helpful depending on characteristics of the objective \cite{choromanski2018structured}.
We see that this is indeed true, the performance of all algorithms improve; GS, BeS, and Orthogonal ES are able to learn.
The algorithm with highest reward changes from Guided ES to BeS-shrinkage and BeS as optimization progresses.

\begin{figure}[t]
\vskip 0.1in
\centering
    \begin{center}
        \includegraphics[width=0.42\textwidth]{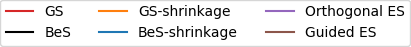}
    \end{center}%
    \begin{subfigure}[b]{0.33\textwidth}
        \includegraphics[width=\textwidth]{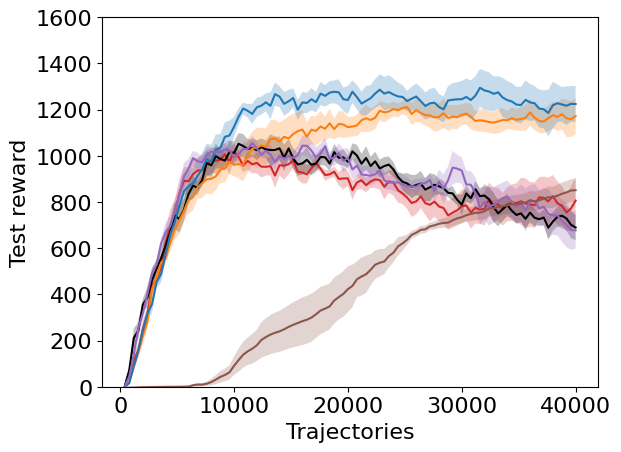}
         \caption{MLP policy}
    \end{subfigure}%

    \begin{subfigure}[b]{0.33\textwidth}
         \includegraphics[width=\textwidth]{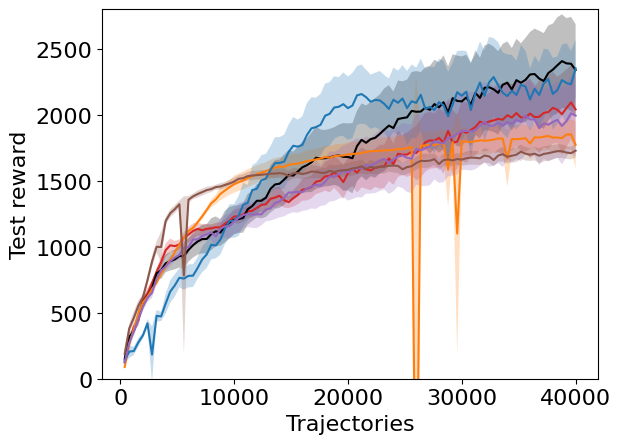}
         \caption{Antithetic gradient estimator}
    \end{subfigure}%
\caption{For Ant with $L=20$}
\label{fig:rlablations}
\vskip -0.1in
\end{figure}

\subsection{DFO benchmarks}\label{sec:dfobenchmark}

Finally, we experiment on the \emph{noisy} benchmark from Nevergrad \cite{nevergrad}, a DFO library.
It consists of four classical minimization objectives, \emph{sphere}, \emph{rosenbrock}, \emph{cigar}, \emph{hm}, with only noisy evaluations available during optimization.
We consider data dimensions $d=10$ or $d=100$ and a range of values of $L$ for computing the gradient estimate at each iteration. 
As in Section \ref{sec:comparisons}, we show results for $N=1$ averaged over five randomly generated seeds.

Figure \ref{fig:benchmark} plots the objective as the optimization progresses for the six algorithms.
In most cases, none of the algorithms were able to minimize \emph{cigar}, possibly because it is too ill-conditioned to find good perturbation directions without very large $L$.
Therefore, we do not include it in our plots.

The qualitative results are similar for all three objectives; GS, BeS, and Orthogonal ES are the best algorithms.
When $L<d$ (first two columns), BeS often outperforms GS and Orthogonal ES at the beginning of optimization, but is overtaken by them at the end of optimization.
When $L=d$ (last column), GS and BeS are not statistically significantly different, and outperform Orthogonal ES, which now behaves similarly to GS-shrinkage and BeS-shrinkage.

\begin{figure*}[h!]
\vskip 0.1in
\includegraphics[width=\textwidth]{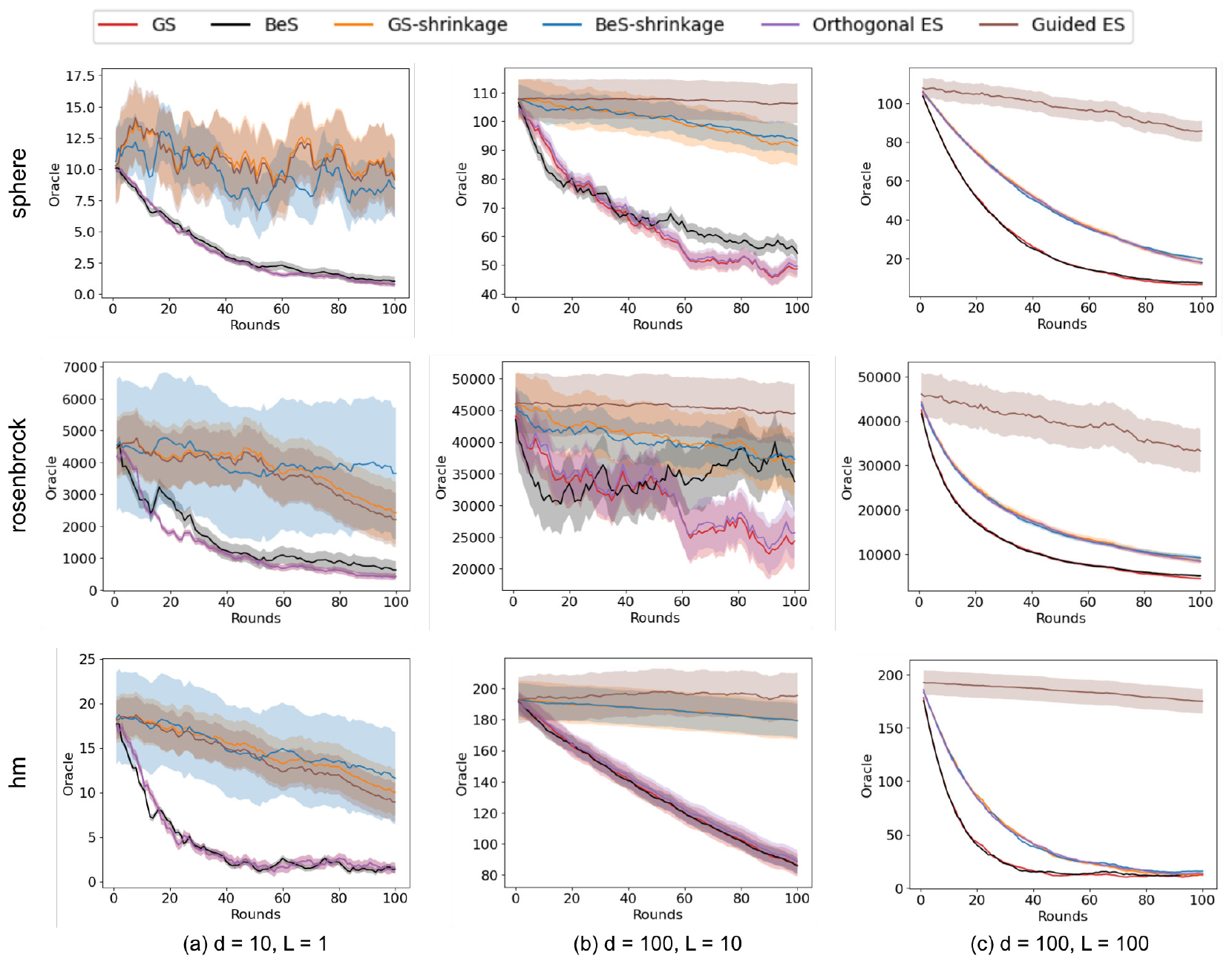}
\caption{For \emph{noisy} DFO benchmark: objective values. When $L=1$, GS is the same as Orthogonal ES.}
\label{fig:benchmark}
\vskip -0.1in
\end{figure*}

\subsection{Discussion}\label{sec:discussion}

The experiments show that by designing the distribution of the directions to minimize the MSE of the gradient, we are often able to obtain superior test performance to GS while remaining as computationally efficient, especially when there are few data points available at each iteration.
Moreover, we are usually competitive with, and in many cases outperform, previously proposed algorithms to improve GS that are more computationally expensive.

One limitation of our work is that GS-shrinkage and BeS-shrinkage do not always outperform GS and BeS, although they have smaller gradient estimate MSE.
We hypothesize that this is due to a form of the bias-variance trade-off.
Theorem \ref{thm:convergencemse} shows that the convergence bound includes the bias of the gradient estimate as well as the MSE in a non-linear and problem-dependent way.
GS and BeS have smaller bias than GS-shrinkage and BeS-shrinkage but larger variance, so it is not surprising that their relative effectiveness is problem-dependent.
Exploring how to adaptively balance the trade-off is an exciting avenue for future work.

The assumptions made highlight other potential directions for future work. We may allow the directions to have dependent entries, learn the distribution type instead of fixing it, or consider how different optimization algorithms may affect the nature of the optimal distribution. We could also design algorithms that, instead of minimizing only the largest term of the MSE \eqref{eq:msepart1}, apply Follow the Regularized Leader \cite{borsos2018online} to minimize the entire MSE over all optimization iterations using estimates of the gradients of the objective and the random evaluation.

\section{Conclusion}\label{sec:conclusion}
In this paper, we generalize Gaussian smoothing to sample directions from arbitrary distributions.
Doing so enables us to choose distributions that minimize the MSE of the gradient estimates and speed up optimization convergence.
We construct three distributions that lead to lower MSE than the standard normal without needing any information about the objective.
Experiments on linear regression confirm our theoretical results and experiments on reinforcement learning and DFO benchmarks show that the derived algorithms often improve over GS.

\section*{Acknowledgements}
We would like to thank the other members of Emergent AI at Intel Labs for providing helpful feedback on the submission.


\bibliography{refs}
\bibliographystyle{icml2022}

\newpage
\appendix
\onecolumn
\section{Proofs for Section \ref{sec:background}}\label{sec:sgdproof}
\subsection{Proof of Theorem \ref{thm:convergencemse}}

Since $F(\theta)$ is differentiable and $\mu$-smooth,
\begin{align*}
    F(\theta^{t+1}) \leq F(\theta^t)-\eta \nabla_\theta F(\theta^t)\trans g(\theta^t)+\dfrac{\mu\eta^2}{2}\|g(\theta^t)\|_2^2.
\end{align*}

For convenience, let $b^t$ and $v^t$ be the bias and variance of $g^t$, respectively.
Taking the expectation with respect to the randomness in $g^t$,
\begin{align*}
    \Ee[F(\theta^{t+1})]&\leq F(\theta^t)-\eta\nabla_\theta F(\theta^t)\trans(\nabla_\theta F(\theta^t)+b^t)+\dfrac{\mu\eta^2}{2}(\|\nabla_\theta F(\theta^t)+b^t\|_2^2+\trace(v^t)) \\
    &= F(\theta^t)-\eta\|\nabla_\theta F(\theta^t)\|_2^2-\eta\nabla_\theta F(\theta^t)\trans b^t+\dfrac{\mu\eta^2}{2}\|\nabla_\theta F(\theta^t)\|_2^2+\dfrac{\mu\eta^2}{2}\|b^t\|_2^2+\mu\eta^2\nabla_\theta F(\theta^t)\trans b^t+\dfrac{\mu\eta^2}{2}\trace(v^t)
\end{align*}
The first inequality uses the definition of variance.

Since $\eta\leq 1/\mu$, we have
\begin{align*}
    \Ee[F(\theta^{t+1})]&\leq F(\theta^t)-\dfrac{\eta}{2}\|\nabla_\theta F(\theta^t)\|_2^2+\dfrac{\mu\eta^2}{2}M+\eta(\mu\eta-1)\nabla_\theta F(\theta^t)\trans b^t \\
    &\leq F(\theta^t)-\dfrac{\eta}{2}\|\nabla_\theta F(\theta^t)\|_2^2+\dfrac{\mu\eta^2}{2}M+\eta B\|\nabla_\theta F(\theta^t)\|_2^2
\end{align*}
where the first line follows from the fact that the MSE of an estimate can be decomposed into the sum of the trace of the variance and the squared norm of the bias \cite{msedecomp} and the second line follows from the assumption on the norm of the bias of $g^t$.

Rearranging, we obtain for $B<\dfrac{1}{2}$
\begin{align*}
    \eta(\dfrac{1}{2}-B)\|\nabla_\theta F(\theta^t)\|_2^2 &\leq F(\theta^t)-\Ee[F(\theta^{t+1})]+\dfrac{\mu\eta^2}{2}M \\
    \dfrac{1}{T}\sum_t\Ee\left(\|\nabla_\theta F(\theta^t)\|_2^2\right) &\leq \dfrac{4\Delta}{\eta(1-2B)T}+\dfrac{\mu\eta M}{1-2B} \\
    &=\dfrac{M+4\Delta\mu}{(1-2B)\sqrt{T}} \quad\text{for $\eta=\dfrac{1}{\mu\sqrt{T}}$.}
\end{align*}

\section{Proofs for Section \ref{sec:algos}}\label{sec:mainproofs}
We first present a lemma that will be useful for subsequent proofs.

\begin{lemma}\label{lem:tracequartic}
Suppose that the $d$ entries of a vector $\epsilon$ are IID samples from a distribution with expectation $0$, variance $\sigma^2$, and kurtosis $k$. For any matrix $A$, $\trace(\Ee_\epsilon(\epsilon\epsilon\trans A\epsilon\epsilon\trans))=\sigma^4(d+k-1)\trace(A)$.
\end{lemma}
\begin{proof}
It suffices to sum the diagonal entries of $\Ee_\epsilon(\epsilon\epsilon\trans A\epsilon\epsilon\trans)$. The $j\th$ entry is
\begin{align*}
    \Ee_\epsilon(\epsilon\epsilon\trans A\epsilon\epsilon\trans)_{jj} &=\Ee_\epsilon(\epsilon_j^2\sum_{}A_{ab}\epsilon_a\epsilon_b) \\
    &=\Ee_\epsilon(\epsilon_j^4)A_{jj}+\Ee_\epsilon(\epsilon_j^2)\Ee_\epsilon(\epsilon_b^2)\sum_{b\neq j}A_{bb} \\
    &=k\sigma^4A_{jj}+\sigma^4(\trace(A)-A_{jj}) = \sigma^4(\trace(A)+(k-1)A_{jj})
\end{align*}
Summing up over $j$,
\begin{align*}
    \trace(\Ee_\epsilon(\epsilon\epsilon\trans A\epsilon\epsilon\trans)) &= \sigma^4(d\trace(A)+(k-1)\trace(A))=\sigma^4(d+k-1)\trace(A)
\end{align*}
\end{proof}

\subsection{Proof of Lemma \ref{thm:fdempmse}}

The MSE of an estimator can be decomposed as the sum of the squared norm of its bias and the trace of its variance \cite{msedecomp}.
Thus, we compute the bias and the variance of $\nabla_\theta \hat{F}^{FD}(\theta)$ separately.

\paragraph{Bias}
\begin{align*}
    \Ee[\nabla_\theta \hat{F}^{FD}(\theta)] &= \Ee_{\epsilon,\xi}\left[\dfrac{1}{cLN}\sum_{l,i} (f(\theta+c\epsilon_l,\xi_i)-f(\theta,\xi_i))\epsilon_l\right] \\
    &= \Ee_\epsilon\left[\dfrac{1}{cL}\sum_l (F(\theta+c\epsilon_l)-F(\theta))\epsilon_l\right] \\
    &= \Ee_\epsilon\left[\dfrac{1}{c}(F(\theta+c\epsilon)-F(\theta))\epsilon\right] \\
    &= \Ee_\epsilon\left[\dfrac{1}{c}(c\nabla_\theta F(\theta)\trans\epsilon+c^2\epsilon\trans h(\theta+c\epsilon)\epsilon)\epsilon\right]
\end{align*}
where the last equality follows from Taylor's Theorem \cite{zeidler1986nonlinear} and $h$ is some scalar-valued function such that $h(y)\to 0$ as $y\to\theta$.
\begin{align*}
    \Ee[\nabla_\theta \hat{F}^{FD}(\theta)] &= \Ee_\epsilon\left[\epsilon(\epsilon\trans\nabla_\theta F(\theta)+c\epsilon\trans h(\theta+c\epsilon)\epsilon)\right] \\
    &= \sigma^2\nabla_\theta F(\theta)+c\Ee_\epsilon[\epsilon\epsilon\trans h(\theta+c\epsilon)\epsilon]
\end{align*}
where $\sigma^2$ is the variance of each entry of $\epsilon$.
If the Dominated Convergence Theorem \cite{billingsley1995probability} holds, i.e. $|\epsilon\epsilon\trans h(\theta+c\epsilon)\epsilon|$ is upper bounded by some integrable function of $\epsilon$, then as $c\to 0$,
\begin{align*}
    \Ee[\nabla_\theta \hat{F}^{FD}(\theta)] &\to \sigma^2\nabla_\theta F(\theta),
\end{align*}
and the squared norm of the bias of $\nabla_\theta \hat{F}^{FD}(\theta)$ is $(\sigma^2-1)^2\|\nabla_\theta F(\theta)\|_2^2$.

\paragraph{Variance}

Using the law of total variance \cite{billingsley1995probability},
\begin{align*}
    \var[\nabla_\theta \hat{F}^{FD}(\theta)] &= \var_\epsilon(\Ee_\xi[\nabla_\theta \hat{F}^{FD}(\theta)\mid\epsilon])+\Ee_\epsilon(\var_\xi[\nabla_\theta \hat{F}^{FD}(\theta)\mid\epsilon]) \\
    &=\var_\epsilon\left[\dfrac{1}{cL}\sum_l(F(\theta+c\epsilon_l)-F(\theta))\epsilon_l\right]+\Ee_\epsilon\left(\dfrac{1}{c^2L}\var_\xi\left[\dfrac{1}{N}\sum_i(f(\theta+c\epsilon,\xi_i)-f(\theta,\xi_i))\epsilon\mid\epsilon\right]\right) \\
    &=\dfrac{1}{c^2L}\var_\epsilon\left[(F(\theta+c\epsilon)-F(\theta))\epsilon\right]+\dfrac{1}{c^2LN}\Ee_\epsilon\left(\var_\xi\left[(f(\theta+c\epsilon,\xi)-f(\theta,\xi))\epsilon\mid\epsilon\right]\right) \\
    &=\dfrac{1}{c^2L}\var_\epsilon\left[\epsilon(c\epsilon\trans\nabla_\theta F(\theta)+c^2\epsilon\trans h(\theta+c\epsilon)\epsilon)\right] \\
    &\quad+\dfrac{1}{c^2LN}\Ee_\epsilon\left(\var_\xi\left[\epsilon(c\epsilon\trans\nabla_\theta f(\theta,\xi)+c^2\epsilon\trans h'(\theta+c\epsilon,\xi)\epsilon)\right]\right) \\
    &=\dfrac{1}{L}\var_\epsilon\left[\epsilon(\epsilon\trans\nabla_\theta F(\theta)+c\epsilon\trans h(\theta+c\epsilon)\epsilon)\right]+\dfrac{1}{LN}\Ee_\epsilon\left(\var_\xi\left[\epsilon(\epsilon\trans\nabla_\theta f(\theta,\xi)+c\epsilon\trans h'(\theta+c\epsilon,\xi)\epsilon)\right]\right)
\end{align*}
where the next-to-last equality follows from Taylor's Theorem and $h'$ is some scalar-valued function with the same condition as $h$.
Assuming that the Dominated Convergence Theorem holds again, as $c\to 0$,
\begin{align*}
    \var[\nabla_\theta \hat{F}^{FD}(\theta)] &\to \dfrac{1}{L}\var_\epsilon\left[\epsilon\epsilon\trans\nabla_\theta F(\theta)\right]+\dfrac{1}{LN}\Ee_\epsilon\left(\var_\xi\left[\epsilon\epsilon\trans\nabla_\theta f(\theta,\xi)\right]\right) \\
    &=\dfrac{1}{L}\Ee_\epsilon\left[\epsilon\epsilon\trans\nabla_\theta F(\theta)\nabla_\theta F(\theta)\trans\epsilon\epsilon\trans\right]-\dfrac{1}{L}\Ee_\epsilon\left[\epsilon\epsilon\trans\nabla_\theta F(\theta)\right]\Ee_\epsilon\left[\epsilon\epsilon\trans\nabla_\theta F(\theta)\right]\trans \\
    &\quad+\dfrac{1}{LN}\Ee_\epsilon\left(\epsilon\epsilon\trans\var_\xi\left[\nabla_\theta f(\theta,\xi)\right]\epsilon\epsilon\trans\right)
\end{align*}

Using Lemma \ref{lem:tracequartic}, as $c\to 0$,
\begin{align*}
    \trace\left(\var[\nabla_\theta \hat{F}^{FD}(\theta)]\right) &\to \dfrac{1}{L}\sigma^4(d+k-1)\trace(\nabla_\theta F(\theta)\nabla_\theta F(\theta)\trans)-\dfrac{1}{L}\trace(\sigma^4\nabla_\theta F(\theta)\nabla_\theta F(\theta)\trans) \\
    &\quad+\dfrac{1}{LN}\sigma^4(d+k-1)\trace(\var_\xi\left[\nabla_\theta f(\theta,\xi)\right]) \\
    &=\dfrac{\sigma^4}{L}(d+k-2)\|\nabla_\theta F(\theta)\|_2^2+\dfrac{\sigma^4}{LN}(d+k-1)\trace(\var_\xi\left[\nabla_\theta f(\theta,\xi)\right])
\end{align*}
Thus, the MSE of $\nabla_\theta \hat{F}^{FD}(\theta)$ is $((\sigma^2-1)^2+\dfrac{\sigma^4}{L}(d+k-2))\|\nabla_\theta F(\theta)\|_2^2+\dfrac{\sigma^4}{LN}(d+k-1)\trace(\var_\xi[\nabla_\theta f(\theta,\xi)])$.

\subsection{Proof of Theorem \ref{thm:gsshrinkage}}

We first make a change of variable.
Let $x=\sigma^2$. Then, \eqref{eq:gsminimize} becomes
\begin{align*}
    &\min_{x>0} O(x)\triangleq (x-1)^2+\dfrac{d+1}{L}x^2.
\end{align*}
Simplifying, $O(x)=(1+\dfrac{d+1}{L})x^2-2x+1$.
It is clear that $O(x)$ is a convex quadratic function, with minimum at $x^*=\dfrac{L}{L+d+1}>0$. Thus, $\sigma^{2*}=\dfrac{L}{L+d+1}$.

By definition, $\sigma^{2*}$ minimizes \eqref{eq:msepart1}. 
Since $\sigma^{2*}<1$ and $k=3$ for both $\nN(0,1)$ and $\nN(0,\sigma^{2*})$, it follows that the value of \eqref{eq:msepart2} is smaller for $\epsilon_{lj}\sim\nN(0,\sigma^{2*})$ than for $\epsilon_{lj}\sim\nN(0,1)$.
Hence, the MSE of $\nabla_\theta \hat{F}^{FD}(\theta)$ is smaller for $\epsilon_{lj}\sim\nN(0,\sigma^{2*})$ than for $\epsilon_{lj}\sim\nN(0,1)$.

\subsection{Proof of Theorem \ref{thm:bsshrinkage}}

We first make a change of variable. 
Let $x=p(1-p)$ and $y=m^2$. Then, \eqref{eq:bsminimize} becomes
\begin{align*}
    &\min_{x\in(0,0.25],y>0} O(x,y)\triangleq \left(\dfrac{x}{y}-1\right)^2+\dfrac{x^2}{y^2L}\left(d+1+\dfrac{1-6x}{x}\right).
\end{align*}
Simplifying,
\begin{align*}
    O(x,y) &= \dfrac{x^2}{y^2}-\dfrac{2x}{y}+1+\dfrac{x^2}{y^2}\dfrac{d-5}{L}+\dfrac{x}{y^2L} \\
    &=\dfrac{x^2}{y^2}\dfrac{L+d-5}{L}+x\dfrac{1-2yL}{y^2L}+1
\end{align*}
which is convex in $x$ for fixed $y$ if $L+d>5$. 
Next, we solve for $x$ in terms of $y$, obtaining $x^*=\dfrac{2yL-1}{2(L+d-5)}$ if $x$ were unconstrained.

However, since $x$ is constrained to $(0,0.25]$, there are three cases:
\begin{enumerate}
    \item If $y\leq\dfrac{1}{2L}$: The lower constraint is active. $x^*=0$ and $O(x^*,y)=1$.
    \item If $y>\dfrac{L+d-3}{4L}$: The upper constraint is active. $x^*=0.25$ and
    \begin{align*}
        O(x^*,y) &= 1+\dfrac{L+d-5}{16y^2L}+\dfrac{1-2yL}{4y^2L} = 1-\dfrac{1}{2y}+\dfrac{L+d-1}{16y^2L} \\
        &<1-\dfrac{1}{2y}+\dfrac{4yL+2}{16y^2L}=1-\dfrac{1}{4y}+\dfrac{1}{8y^2L}
    \end{align*}
    where the inequality in the second line follows from the condition on $y$.
    \item If $y>\dfrac{1}{2L}$ and $y\leq\dfrac{L+d-3}{4L}$: Neither constraint is active. $x^*=\dfrac{2yL-1}{2(L+d-5)}$ and
    \begin{align*}
        O(x^*,y) &= 1-\dfrac{(1-2yL)^2}{y^4L^2}\Bigg/\dfrac{4(L+d-5)}{y^2L} = 1-\dfrac{(1-2yL)^2}{4y^2L(L+d-5)} \\
        &\geq 1-\dfrac{(2yL-1)^2}{4y^2L(4yL-2)}=1-\dfrac{2yL-1}{8y^2L}=1-\dfrac{1}{4y}+\dfrac{1}{8y^2L}
    \end{align*}
    where the inequality in the second line follows from $y\leq\dfrac{L+d-3}{4L}$.
\end{enumerate}
In cases $2$ and $3$, $y>\dfrac{1}{2L}$, so $1-\dfrac{1}{4y}+\dfrac{1}{8y^2L}<1$.
Therefore, the smallest possible value of $O(x^*,y)$ occurs in case $2$, with $x^*=0.25$, $y>\dfrac{L+d-3}{4L}$, and $Q(y)=O(x^*,y)=1-\dfrac{1}{2y}+\dfrac{L+d-1}{16y^2L}$.

Finally, we minimize $Q(y)$ restricted to $y>\dfrac{L+d-3}{4L}$.
Observe that $Q(y)$ is a convex quadratic function of $\dfrac{1}{y}$, so $\dfrac{1}{y^*}=\dfrac{4L}{L+d-1}$, or $y^*=\dfrac{L+d-1}{4L}$.

Thus, $x^*=0.25$ and $y^*=\dfrac{L+d-1}{4L}$, corresponding to $p^*=0.5$ and $m^*=\sqrt{\dfrac{L+d-1}{4L}}$.

By definition, $p^*=0.5$ and $m^*=\sqrt{\dfrac{L+d-1}{4L}}$ minimizes \eqref{eq:msepart1}. 
The kurtosis of $\dfrac{B_p^*-p^*}{m^*}$ is the same as that of $\dfrac{B_{0.5}^*-0.5}{0.5}$; however, it has smaller variance since $m^*>0.5$.
Hence, the MSE of $\nabla_\theta\hat{F}^{FD}(\theta)$ is smaller for $\epsilon_{lj}\sim\dfrac{B_p^*-p^*}{m^*}$ than for $\epsilon_{lj}\sim\dfrac{B_{0.5}^*-0.5}{0.5}$.

\subsection{GS-shrinkage or BeS-shrinkage?}\label{sec:gsvsbsshrinkage}

To analyze whether GS-shrinkage or BeS-shrinkage is superior, we compare their MSEs for the gradient estimator \eqref{eq:fdempirical}, using Lemma \ref{thm:fdempmse}.
The MSE for GS-shrinkage is
\begin{align*}
    MSE(GSs) &= \left(\dfrac{(d+1)^2}{(L+d+1)^2}+\dfrac{L^2(d+1)}{L(L+d+1)^2}\right)\|\nabla_\theta F(\theta)\|_2^2+\dfrac{L^2(d+2)}{LN(L+d+1)^2}\trace(\var_\xi[\nabla_\theta f(\theta,\xi)]) \\
    &= \dfrac{(d+1)^2+L(d+1)}{(L+d+1)^2}\|\nabla_\theta F(\theta)\|_2^2+\dfrac{L(d+2)}{N(L+d+1)^2}\trace(\var_\xi[\nabla_\theta f(\theta,\xi)])
\end{align*}
The MSE for BeS-shrinkage is
\begin{align*}
    MSE(BeSs) &= \left(\left(1-\dfrac{4L}{4(L+d-1)}\right)^2+\dfrac{16L^2(d-1)}{16L(L+d-1)^2}\right)\|\nabla_\theta F(\theta)\|_2^2 \\
    &\quad+\dfrac{d(16L^2)}{16LN(L+d-1)^2}\trace(\var_\xi[\nabla_\theta f(\theta,\xi)]) \\
    &=\left(\dfrac{(d-1)^2}{(L+d-1)^2}+\dfrac{L(d-1)}{(L+d-1)^2}\right)\|\nabla_\theta F(\theta)\|_2^2+\dfrac{dL}{N(L+d-1)^2}\trace(\var_\xi[\nabla_\theta f(\theta,\xi)])
\end{align*}

Therefore,
\begin{align*}
    & MSE(GSs)-MSE(BeSs) \\
    &= \|\nabla_\theta F(\theta)\|_2^2\left(\dfrac{d+1}{L+d+1}-\dfrac{d-1}{L+d-1}\right)+\dfrac{\trace(\var_\xi[\nabla_\theta f(\theta,\xi)])L}{N}\left(\dfrac{d+2}{(L+d+1)^2}-\dfrac{d}{(L+d-1)^2}\right) \\
    &=\|\nabla_\theta F(\theta)\|_2^2\dfrac{(d+1)(L+d-1)-(d-1)(L+d+1)}{(L+d+1)(L+d-1)} \\
    &\quad+\trace(\var_\xi[\nabla_\theta f(\theta,\xi)])\dfrac{L}{N}\dfrac{(d+2)(L+d-1)^2-d(L+d+1)^2}{(L+d+1)^2(L+d-1)^2} \\
    &= \|\nabla_\theta F(\theta)\|_2^2\dfrac{2L}{(L+d+1)(L+d-1)}+\trace(\var_\xi[\nabla_\theta f(\theta,\xi)])\dfrac{2L}{N}\dfrac{L^2-2L+2-(d+1)^2}{(L+d+1)^2(L+d-1)^2} \\
    &= \dfrac{2L}{(L+d-1)(L+d+1)}\left(\|\nabla_\theta F(\theta)\|_2^2+\dfrac{L^2-2L+2-(d+1)^2}{N(L+d-1)(L+d+1)}\trace(\var_\xi[\nabla_\theta f(\theta,\xi)])\right)
\end{align*}

At the beginning of optimization, $\|\nabla_\theta F(\theta)\|_2^2$ is large.
If it is large compared to $\trace(\var_\xi[\nabla_\theta f(\theta,\xi)])/N$, since $L^2-2L+2-(d+1)^2<0$ for high-dimensional problems, $MSE(GSs)>MSE(BeSs)$.

At the end of optimization, $\|\nabla_\theta F(\theta)\|_2^2\approx 0$. 
Then, since $L^2-2L+2-(d+1)^2<0$ for high-dimensional problems, $MSE(GSs)<MSE(BeSs)$.

\subsection{Algorithms for the antithetic gradient estimator}\label{sec:antithetic}

Suppose that instead of the forward difference gradient estimator, we use the antithetic gradient estimator.
Mathematically, given $N$ samples from the oracle $\xi_i$ and $L$ IID sampled directions $\epsilon_l$, let 
\begin{align}\label{eq:atempirical}
    \nabla_\theta \hat{F}^{AT}(\theta) &= \dfrac{1}{2cLN}\sum_{l,i} (f(\theta+c\epsilon_l,\xi_i)-f(\theta-c\epsilon_l,\xi_i))\epsilon_l
\end{align}
Under Assumption \ref{assume:perturb}, we compute the MSE of $\nabla_\theta\hat{F}^{AT}(\theta)$, with the same strategy as for $\nabla_\theta\hat{F}^{FD}(\theta)$ (Lemma \ref{thm:fdempmse}).

\paragraph{Bias}
\begin{align*}
    \Ee[\nabla_\theta \hat{F}^{AT}(\theta)] &= \Ee_{\epsilon,\xi}\left[\dfrac{1}{2cLN}\sum_{l,i} (f(\theta+c\epsilon_l,\xi_i)-f(\theta-c\epsilon_l,\xi_i))\epsilon_l\right] \\
    &= \Ee_\epsilon\left[\dfrac{1}{2cL}\sum_l (F(\theta+c\epsilon_l)-F(\theta-c\epsilon_l))\epsilon_l\right] \\
    &= \Ee_\epsilon\left[\dfrac{1}{2c}(F(\theta+c\epsilon)-F(\theta-c\epsilon_l))\epsilon\right] \\
    &= \Ee_\epsilon\left[\dfrac{1}{2c}(2c\nabla_\theta F(\theta)\trans\epsilon+c^3\sum_{|\alpha|_1=3} h_\alpha(\theta+c\epsilon)\epsilon^\alpha)\epsilon\right]
\end{align*}
where the last equality follows from Taylor's Theorem, $\alpha$ is a $d$-dimensional vector of non-negative integers, $\epsilon^\alpha$ indicates element-wise power, and $h_\alpha$ is some scalar-valued function such that $h_\alpha(y)\to 0$ as $y\to\theta$.
\begin{align*}
    \Ee[\nabla_\theta \hat{F}^{AT}(\theta)] &= \Ee_\epsilon\left[\epsilon(\epsilon\trans\nabla_\theta F(\theta)+c^2\sum_{|\alpha|_1=3} h_\alpha(\theta+c\epsilon)\epsilon^\alpha)\right] \\
    &= \sigma^2\nabla_\theta F(\theta)+c^2\Ee_\epsilon[\epsilon\sum_{|\alpha|_1=3} h_\alpha(\theta+c\epsilon)\epsilon^\alpha]
\end{align*}
where $\sigma^2$ is the variance of each entry of $\epsilon$.
If the Dominated Convergence Theorem \cite{billingsley1995probability} holds, then as $c\to 0$,
\begin{align*}
    \Ee[\nabla_\theta \hat{F}^{AT}(\theta)] &\to \sigma^2\nabla_\theta F(\theta),
\end{align*}
and the squared norm of the bias of $\nabla_\theta \hat{F}^{AT}(\theta)$ is the same as that of $\nabla_\theta \hat{F}^{FD}(\theta)$.

\paragraph{Variance}

Using the law of total variance \cite{billingsley1995probability},
\begin{align*}
    \var[\nabla_\theta \hat{F}^{AT}(\theta)] &= \var_\epsilon(\Ee_\xi[\nabla_\theta \hat{F}^{AT}(\theta)\mid\epsilon])+\Ee_\epsilon(\var_\xi[\nabla_\theta \hat{F}^{AT}(\theta)\mid\epsilon]) \\
    &=\var_\epsilon\left[\dfrac{1}{2cL}\sum_l(F(\theta+c\epsilon_l)-F(\theta-c\epsilon_l))\epsilon_l\right] \\
    &\quad+\Ee_\epsilon\left(\dfrac{1}{4c^2L}\var_\xi\left[\dfrac{1}{N}\sum_i(f(\theta+c\epsilon,\xi_i)-f(\theta-c\epsilon,\xi_i))\epsilon\mid\epsilon\right]\right) \\
    &=\dfrac{1}{4c^2L}\var_\epsilon\left[(F(\theta+c\epsilon)-F(\theta-c\epsilon))\epsilon\right]+\dfrac{1}{4c^2LN}\Ee_\epsilon\left(\var_\xi\left[(f(\theta+c\epsilon,\xi)-f(\theta-c\epsilon,\xi))\epsilon\mid\epsilon\right]\right) \\
    &=\dfrac{1}{4c^2L}\var_\epsilon\left[\epsilon(2c\epsilon\trans\nabla_\theta F(\theta)+c^2\sum_{|\alpha|_1=3} h_\alpha(\theta+c\epsilon)\epsilon^\alpha)\right] \\
    &\quad+\dfrac{1}{4c^2LN}\Ee_\epsilon\left(\var_\xi\left[\epsilon(2c\epsilon\trans\nabla_\theta f(\theta,\xi)+c^2\sum_{|\alpha|_1=3} h'_\alpha(\theta+c\epsilon)\epsilon^\alpha)\right]\right) \\
    &=\dfrac{1}{L}\var_\epsilon\left[\epsilon(\epsilon\trans\nabla_\theta F(\theta)+c\sum_{|\alpha|_1=3} h_\alpha(\theta+c\epsilon)\epsilon^\alpha)\right]\\
    &\quad+\dfrac{1}{LN}\Ee_\epsilon\left(\var_\xi\left[\epsilon(\epsilon\trans\nabla_\theta f(\theta,\xi)+c\sum_{|\alpha|_1=3} h'_\alpha(\theta+c\epsilon)\epsilon^\alpha)\right]\right)
\end{align*}
where the next-to-last equality follows from Taylor's Theorem and $h'$ is some scalar-valued function with the same condition as $h$.
Assuming that the Dominated Convergence Theorem holds again, as $c\to 0$, the limit of $\var[\nabla_\theta \hat{F}^{AT}(\theta)]$ is the same as that of $\var[\nabla_\theta \hat{F}^{FD}(\theta)]$.

Thus, the MSE of $\nabla_\theta \hat{F}^{AT}(\theta)$ is the same as that of $\var[\nabla_\theta \hat{F}^{FD}(\theta)]$.
It follows that any algorithm to minimize the MSE of $\nabla_\theta \hat{F}^{AT}(\theta)$ must be the same as an algorithm to minimize the MSE of $\nabla_\theta \hat{F}^{FD}(\theta)$.

\section{Experimental details}\label{sec:expdetails}
\subsection{Validating theory on linear regression}\label{sec:linregdetails}

\mypara{Computing the gradient of the objective}
The objective is the squared error loss. For our data model \eqref{eq:linregmodel}, it is
\begin{align*}
    F(\theta) &= \Ee[(y-\theta\trans x)^2/2] = \Ee[(\gamma\trans x-\theta\trans x+\epsilon)^2/2] \\
    &=\Ee[\gamma\trans xx\trans\gamma/2+\theta\trans xx\trans\theta/2+\epsilon^2/2-\gamma\trans xx\trans\theta-\theta\trans x\epsilon+\gamma\trans x\epsilon] \\\
    &\equiv \Ee[\theta\trans xx\trans\theta/2-\gamma\trans xx\trans\theta-\theta\trans x\epsilon] \\
    &=\Ee[\theta\trans \QQ\theta/2-\gamma\trans\QQ\theta] = \theta\trans\Ee[\QQ]\theta/2-\Ee[\gamma\trans\QQ]\theta
\end{align*}
where in the third line we have ignored terms that do not include $\theta$.

The gradient of the objective is $\nabla_\theta F(\theta) = \Ee(\QQ)\theta-\Ee(\QQ\gamma)=\theta-\Ee(\QQ\gamma)$, since
\begin{align*}
    \Ee(\QQ) &= \Ee_\VV[\Ee_\gamma(\VV\diag(\gamma)\VV\trans\mid\VV)]=\Ee_\VV[\VV\Ee_\gamma(\diag(\gamma))\VV\trans] \\
    &=\Ee_\VV[\VV\VV\trans] = \II \quad \text{since $\Ee(\gamma)$ is a vector of ones $1_d$ and $\VV$ is orthogonal}
\end{align*}

Prior to the start of optimization, $\Ee(\QQ\gamma)$ is estimated via Monte Carlo with $1000$ samples. 
The estimate is plugged into $\nabla_\theta F(\theta)$ to obtain the gradient of the objective.

\mypara{Optimization and testing}
We first sample $1000$ data points from \eqref{eq:linregmodel} to serve as the test set and initialize the parameters $\theta$ by sampling from $\nN(0,\II)$.
There are $100$ rounds. 
Each round consists of i) $10$ optimization iterations of SGD with the gradient estimated from \eqref{eq:fdempirical} on $N$ newly sampled data points from \eqref{eq:linregmodel} and $L$ newly sampled directions $\epsilon_l$ ii) computation of the squared error loss over the test set.
The MSE of the gradient estimate is computed at each iteration and the average is taken over the $10$ iterations per round.
Note that $f(\theta,\xi_i)$ is the squared error loss on data point $i$.

\mypara{Hyperparameter search}
We ran this experiment for $L=\{2, 6, 20\}$ and $N=\{5, 15, 50\}$, and a selection was shown in the main paper due to space constraints.
Hyperparameters are the spacing $c$, chosen from $\{0.01, 0.1\}$, and the SGD learning rate $\eta$, chosen from $\{0.001, 0.01, 0.1\}$. 
The values chosen are the ones that minimize the test loss at the end of the $100$ rounds, averaged over $3$ randomly generated seeds different from those used in Figure \ref{fig:linreg}.
Tables \ref{tab:linreg_hyper} and \ref{tab:linregshrinkage_hyper} show the chosen hyperparameters for each algorithm and combination of $L$ and $N$.

\begin{table}[t]
\caption{Hyperparameters for GS and BeS in linear regression.}
\label{tab:linreg_hyper}
\vskip 0.1in
\begin{center}
\begin{small}
\begin{sc}
\begin{tabular}{lccr}
\toprule
$(L,N)$ & $c$ & learning rate \\
\midrule
$(2,5)$ & $0.01$ & $0.001$  \\ 
$(6,5)$ & $0.01$ & $0.001$ \\ 
$(20,5)$ & $0.01$ & $0.001$ \\ 
$(2,15)$ & $0.01$ & $0.001$ \\ 
$(6,15)$ & $0.01$ & $0.001$ \\ 
$(20,15)$ & $0.01$ & $0.01$ \\ 
$(2,50)$ & $0.01$ & $0.001$ \\ 
$(6,50)$ & $0.01$ & $0.01$ \\ 
$(20,50)$ & $0.01$ & $0.01$ \\ 
\bottomrule
\end{tabular}
\end{sc}
\end{small}
\quad
\begin{small}
\begin{sc}
\begin{tabular}{lccr}
\toprule
$(L,N)$ & $c$ & learning rate \\
\midrule
$(2,5)$ & $0.01$ & $0.001$ \\ 
$(6,5)$ & $0.1$ & $0.001$ \\ 
$(20,5)$ & $0.01$ & $0.001$ \\ 
$(2,15)$ & $0.01$ & $0.001$ \\ 
$(6,15)$ & $0.1$ & $0.001$ \\ 
$(20,15)$ & $0.01$ & $0.01$ \\ 
$(2,50)$ & $0.01$ & $0.001$ \\ 
$(6,50)$ & $0.01$ & $0.01$ \\ 
$(20,50)$ & $0.01$ & $0.01$ \\ 
\bottomrule
\end{tabular}
\end{sc}
\end{small}
\end{center}
\vskip -0.1in
\end{table}

\begin{table}[t]
\caption{Hyperparameters for GS-shrinkage and BeS-shrinkage in linear regression.}
\label{tab:linregshrinkage_hyper}
\vskip 0.1in
\begin{center}
\begin{small}
\begin{sc}
\begin{tabular}{lccr}
\toprule
$(L,N)$ & $c$ & learning rate \\
\midrule
$(2,5)$ & $0.01$ & $0.01$ \\ 
$(6,5)$ & $0.1$ & $0.01$ \\ 
$(20,5)$ & $0.01$ & $0.01$ \\ 
$(2,15)$ & $0.01$ & $0.1$ \\ 
$(6,15)$ & $0.01$ & $0.1$ \\ 
$(20,15)$ & $0.1$ & $0.01$ \\ 
$(2,50)$ & $0.01$ & $0.1$ \\ 
$(6,50)$ & $0.1$ & $0.1$ \\ 
$(20,50)$ & $0.01$ & $0.1$ \\ 
\bottomrule
\end{tabular}
\end{sc}
\end{small}
\quad
\begin{small}
\begin{sc}
\begin{tabular}{lccr}
\toprule
$(L,N)$ & $c$ & learning rate \\
\midrule
$(2,5)$ & $0.01$ & $0.01$ \\ 
$(6,5)$ & $0.1$ & $0.01$ \\ 
$(20,5)$ & $0.01$ & $0.01$ \\ 
$(2,15)$ & $0.01$ & $0.1$ \\ 
$(6,15)$ & $0.1$ & $0.1$ \\ 
$(20,15)$ & $0.1$ & $0.01$ \\ 
$(2,50)$ & $0.01$ & $0.1$ \\ 
$(6,50)$ & $0.01$ & $0.1$ \\ 
$(20,50)$ & $0.01$ & $0.1$ \\ 
\bottomrule
\end{tabular}
\end{sc}
\end{small}
\end{center}
\vskip -0.1in
\end{table}

\mypara{Full results}
Figures \ref{fig:linregmse_full} and \ref{fig:linregloss_full} contain the complete results for the MSE of the gradient estimate and test loss, respectively.
We see that in all cases, the MSE of the gradient is substantially smaller for GS-shrinkage and BeS-shrinkage than GS and BeS.
The story is less clear for the test loss, but in $5$ out of $9$ cases the test loss of GS-shrinkage and BeS-shrinkage is lower than that of GS and BeS at the end of optimization.
See Section \ref{sec:validation} for further discussion.

\begin{figure}[ht]
\vskip 0.1in
    \begin{center}
        \includegraphics[width=0.6\textwidth]{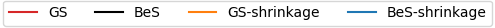}
    \end{center}%
    \begin{subfigure}[b]{0.33\textwidth}
         \includegraphics[width=\textwidth]{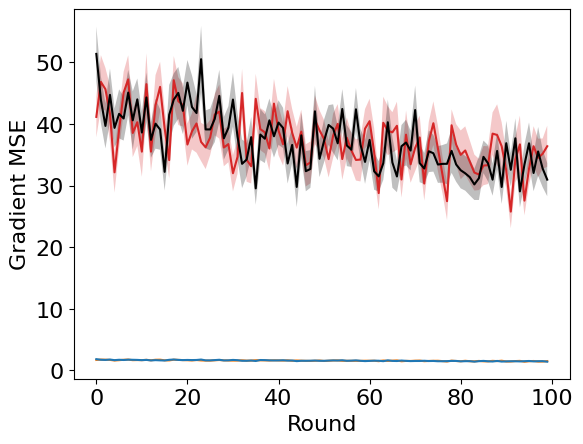}
         \caption{$L=2, N=5$}
    \end{subfigure}
    \hfill
    \begin{subfigure}[b]{0.33\textwidth}
         \includegraphics[width=\textwidth]{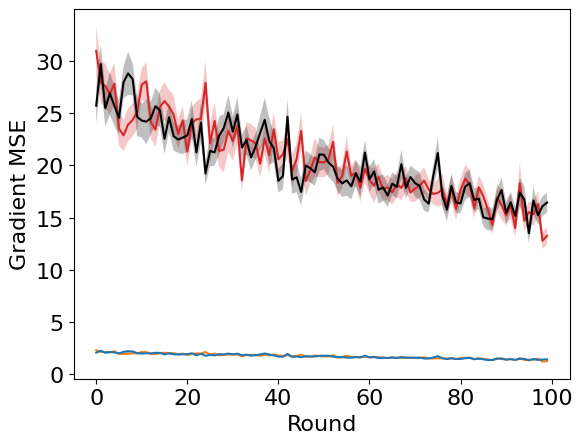}
         \caption{$L=6, N=5$}
    \end{subfigure}
    \hfill
    \begin{subfigure}[b]{0.33\textwidth}
         \includegraphics[width=\textwidth]{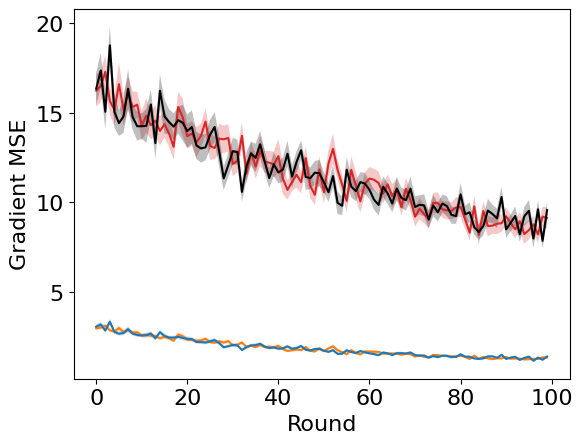}
         \caption{$L=20, N=5$}
    \end{subfigure}%
    
    \begin{subfigure}[b]{0.33\textwidth}
         \includegraphics[width=\textwidth]{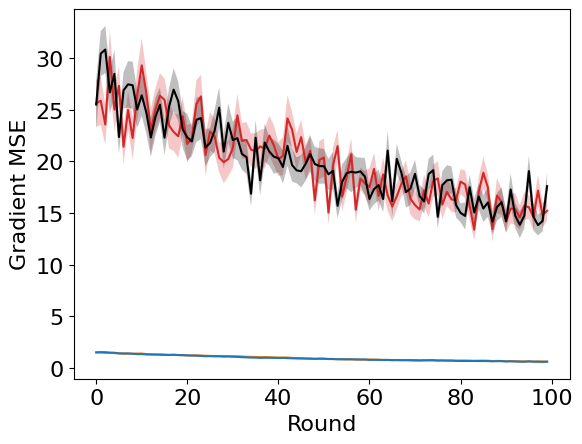}
         \caption{$L=2, N=15$}
    \end{subfigure}
    \hfill
    \begin{subfigure}[b]{0.33\textwidth}
         \includegraphics[width=\textwidth]{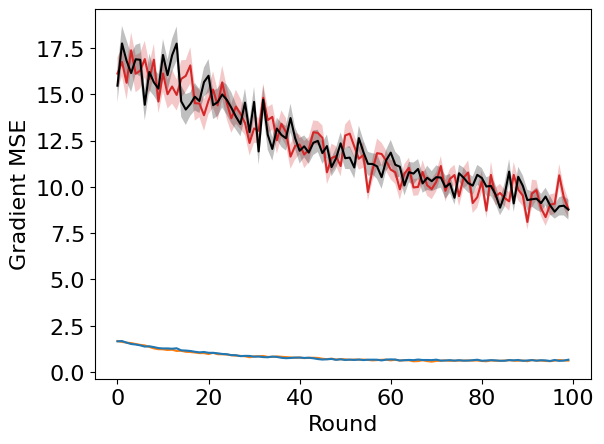}
         \caption{$L=6, N=15$}
    \end{subfigure}
    \hfill
    \begin{subfigure}[b]{0.33\textwidth}
         \includegraphics[width=\textwidth]{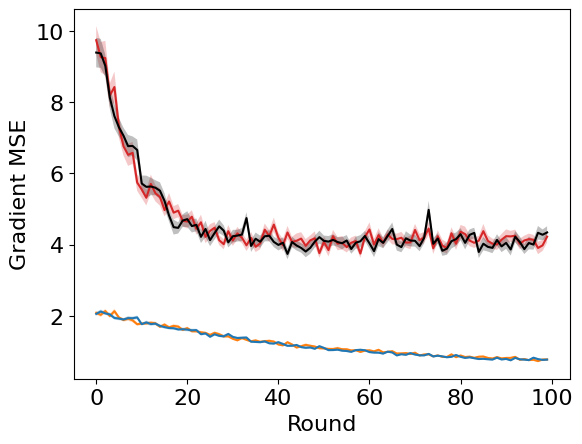}
         \caption{$L=20, N=15$}
    \end{subfigure}%
    
    \begin{subfigure}[b]{0.33\textwidth}
         \includegraphics[width=\textwidth]{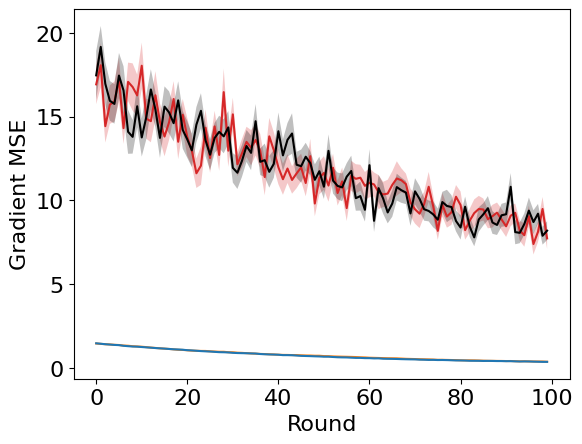}
         \caption{$L=2, N=50$}
    \end{subfigure}
    \hfill
    \begin{subfigure}[b]{0.33\textwidth}
         \includegraphics[width=\textwidth]{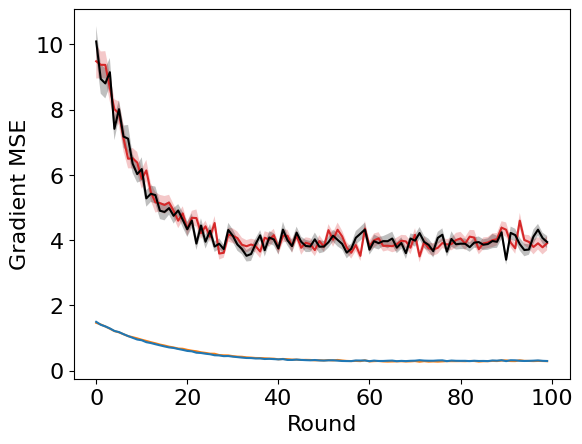}
         \caption{$L=6, N=50$}
    \end{subfigure}
    \hfill
    \begin{subfigure}[b]{0.33\textwidth}
         \includegraphics[width=\textwidth]{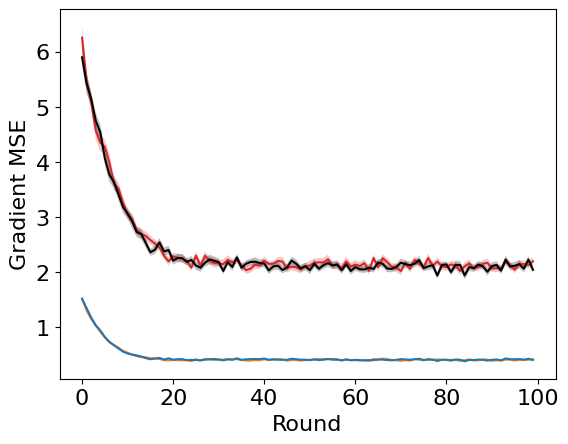}
         \caption{$L=20, N=50$}
    \end{subfigure}
\caption{For linear regression with various $L$ and $N$: MSE of the gradient at each round averaged over the $10$ iterations.}
\label{fig:linregmse_full}
\vskip -0.1in
\end{figure}

\begin{figure}[ht]
\vskip 0.1in
    \begin{center}
        \includegraphics[width=0.6\textwidth]{linregfigs/legend.png}
    \end{center}%
    \begin{subfigure}[b]{0.33\textwidth}
         \includegraphics[width=\textwidth]{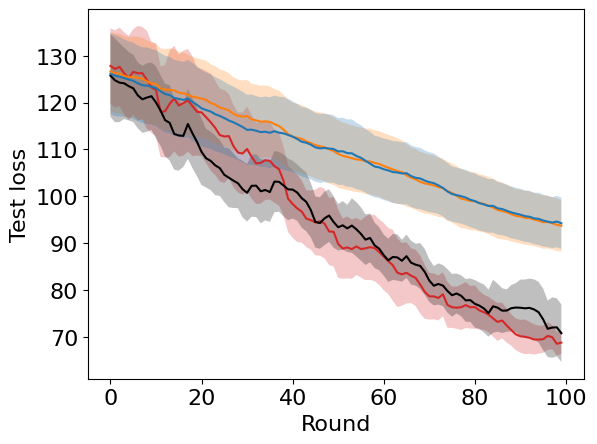}
         \caption{$L=2, N=5$}
    \end{subfigure}
    \hfill
    \begin{subfigure}[b]{0.33\textwidth}
         \includegraphics[width=\textwidth]{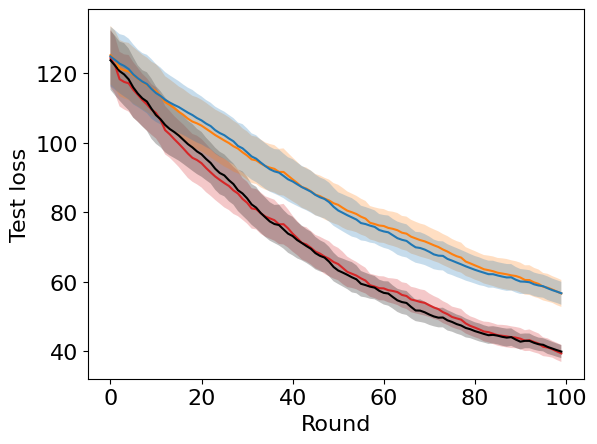}
         \caption{$L=6, N=5$}
    \end{subfigure}
    \hfill
    \begin{subfigure}[b]{0.33\textwidth}
         \includegraphics[width=\textwidth]{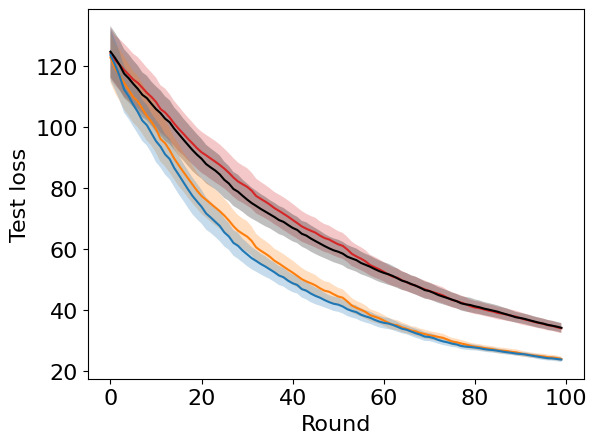}
         \caption{$L=20, N=5$}
    \end{subfigure}%
    
    \begin{subfigure}[b]{0.33\textwidth}
         \includegraphics[width=\textwidth]{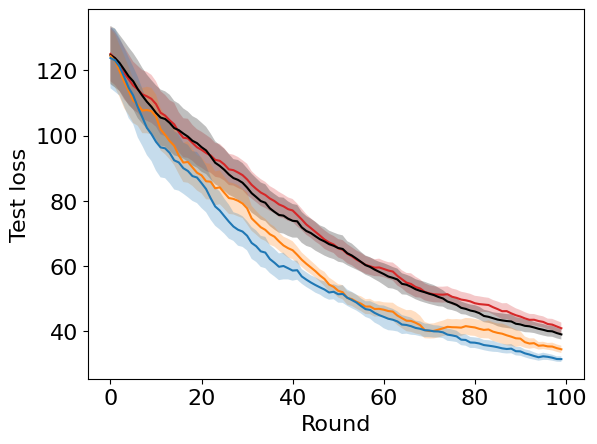}
         \caption{$L=2, N=15$}
    \end{subfigure}
    \hfill
    \begin{subfigure}[b]{0.33\textwidth}
         \includegraphics[width=\textwidth]{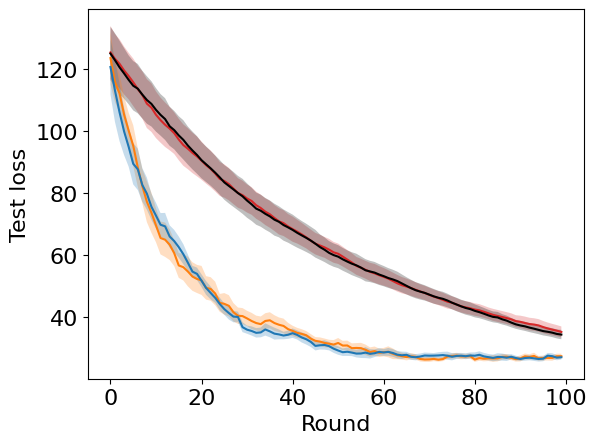}
         \caption{$L=6, N=15$}
    \end{subfigure}
    \hfill
    \begin{subfigure}[b]{0.33\textwidth}
         \includegraphics[width=\textwidth]{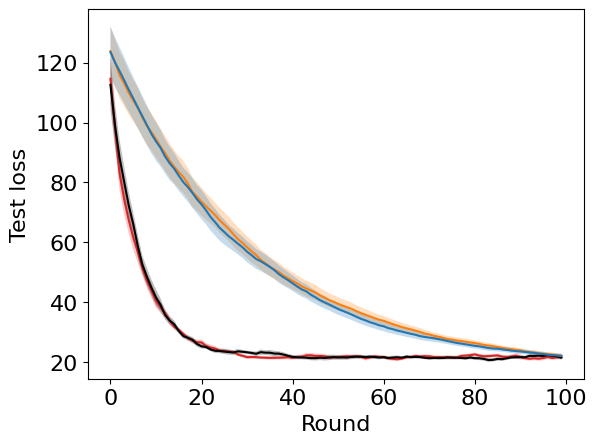}
         \caption{$L=20, N=15$}
    \end{subfigure}%
    
    \begin{subfigure}[b]{0.33\textwidth}
         \includegraphics[width=\textwidth]{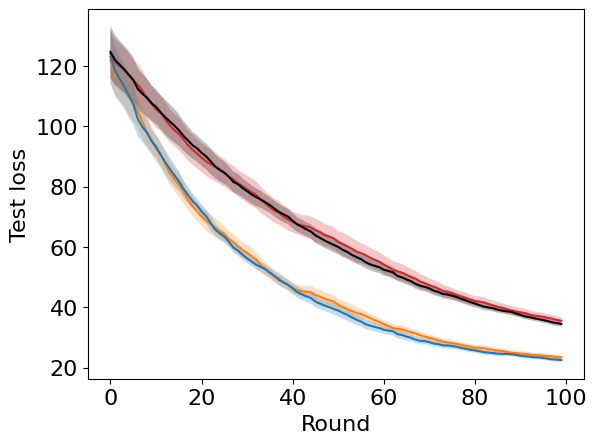}
         \caption{$L=2, N=50$}
    \end{subfigure}
    \hfill
    \begin{subfigure}[b]{0.33\textwidth}
         \includegraphics[width=\textwidth]{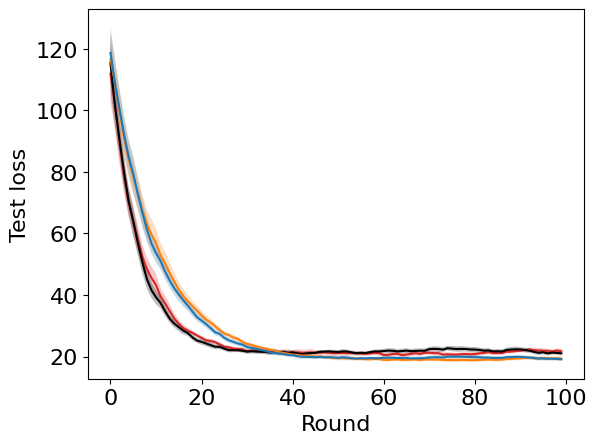}
         \caption{$L=6, N=50$}
    \end{subfigure}
    \hfill
    \begin{subfigure}[b]{0.33\textwidth}
         \includegraphics[width=\textwidth]{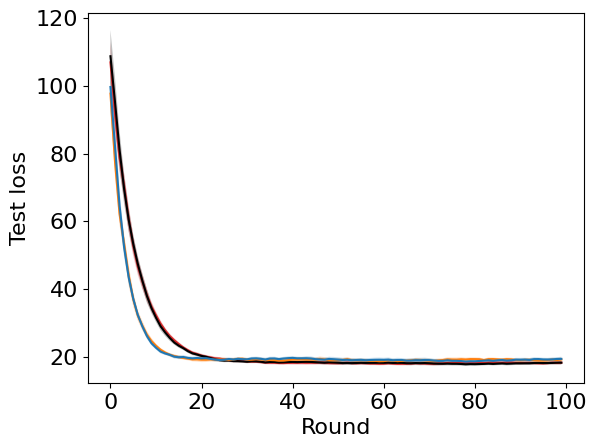}
         \caption{$L=20, N=50$}
    \end{subfigure}
\caption{For linear regression with various $L$ and $N$: Test loss at each round.}
\label{fig:linregloss_full}
\vskip -0.1in
\end{figure}

\subsection{Comparing to baselines on RL}\label{sec:rldetails}

\mypara{Baselines}
Orthogonal ES is the same as GS, but with an application of the Gram-Schmidt process to the directions after they are sampled.
Guided ES samples directions from the distribution $\nN(0,\mathbf{\Sigma})$, where $\mathbf{\Sigma}=\dfrac{\alpha}{d}\II+\dfrac{1-\alpha}{k}\UU\UU\trans$ and $\UU$ is an orthonormal basis for the $k$ previous gradient estimates; computing the basis also requires the Gram-Schmidt process.
Following recommendations in \citet{guidedes} and \citet{lmrs}, we set $\alpha=0.5$ and $k=50$ ($k=10$ when $d<50$) and let $\alpha=1$ for the first $k$ iterations.

\mypara{Optimization and testing}
Our code roughly follows the same structure as \citet{mania2018simple}, parallelizing trajectory generation and standardizing the observations. 
The parameters of the linear policy $\theta$ is initialized at zero. 
There are $100$ rounds, each consisting of $10$ optimization iterations and one test step.
In more detail, every optimization iteration has the following steps:
\begin{enumerate}
    \item Sample $L$ directions $\epsilon_l$.
    \item For each direction, reinitialize the environment, generate one trajectory using the parameters $\theta+c\epsilon_l$ and another using the parameters $\theta$. For those environments with a survival bonus, remove it.
    \item Using the $2L$ rewards, compute the gradient estimate \eqref{eq:fdempirical} with $N=1$, dividing by the standard deviation of the rewards.
    \item Take a gradient ascent step on $\theta$ with learning rate $\eta$.
\end{enumerate}
and each test step has the following steps:
\begin{enumerate}
    \item For $1000$ trials: Reinitialize the environment and generate a trajectory. Record the total reward.
    \item Compute the average and standard deviation of the reward over the trials.
\end{enumerate}

\mypara{Hyperparameter search}
Ant and Walker2d have horizon $1000$, ML1-Reach $150$, and HalfCheetahRandVel $200$.
We ran this experiment for $L=\{2, 6, 20\}$.
Hyperparameters are the spacing $c$, chosen from $\{0.01,0.1\}$, and the learning rate $\eta$, chosen from $\{0.0001, 0.001, 0.01\}$.
The values chosen are the ones that maximize the test reward at the end of the $100$ rounds, averaged over $3$ randomly generated seeds different from those used in Figure \ref{fig:rl}.
Tables \ref{tab:ant_hyper} -- \ref{tab:ml1reach_hyper} show the chosen hyperparameters for each algorithm in the four environments discussed in the main paper.

\begin{table}[t]
\caption{$c$ and learning rate for Ant.}
\label{tab:ant_hyper}
\vskip 0.1in
\begin{center}
\begin{small}
\begin{sc}
\begin{tabular}{lcccr}
\toprule
Algorithm & $L=2$ & $L=6$ & $L=20$ \\ 
\midrule
GS & $0.1$ & $0.1$ & $0.1$ \\ 
BeS & $0.1$ & $0.1$ & $0.1$ \\ 
GS-shrinkage & $0.1$ & $0.1$ & $0.01$ \\ 
BeS-shrinkage & $0.1$ & $0.01$ & $0.01$ \\ 
Orthogonal ES & $0.1$ & $0.1$ & $0.1$ \\ 
Guided ES & $0.1$ & $0.1$ & $0.1$ \\ 
\bottomrule
\end{tabular}
\end{sc}
\end{small}
\quad
\begin{small}
\begin{sc}
\begin{tabular}{lcccr}
\toprule
Algorithm & $L=2$ & $L=6$ & $L=20$ \\ %
\midrule
GS & $0.0001$ & $0.0001$ & $0.0001$ \\ %
BeS & $0.0001$ & $0.0001$ & $0.0001$ \\ %
GS-shrinkage & $0.001$ & $0.001$ & $0.0001$ \\ %
BeS-shrinkage & $0.001$ & $0.0001$ & $0.0001$ \\ %
Orthogonal ES & $0.0001$ & $0.0001$ & $0.0001$ \\ %
Guided ES & $0.001$ & $0.001$ & $0.001$ \\ %
\bottomrule
\end{tabular}
\end{sc}
\end{small}
\end{center}
\vskip -0.1in
\end{table}

\begin{table}[t]
\caption{$c$ and learning rate for Walker2d.}
\label{tab:walker_hyper}
\vskip 0.1in
\begin{center}
\begin{small}
\begin{sc}
\begin{tabular}{lcccr}
\toprule
Algorithm & $L=2$ & $L=6$ & $L=20$ \\ 
\midrule
GS & $0.01$ & $0.01$ & $0.01$ \\ 
BeS & $0.1$ & $0.01$ & $0.01$ \\ 
GS-shrinkage & $0.1$ & $0.01$ & $0.01$ \\ 
BeS-shrinkage & $0.1$ & $0.01$ & $0.1$ \\ 
Orthogonal ES & $0.01$ & $0.01$ & $0.01$ \\ 
Guided ES & $0.1$ & $0.1$ & $0.01$ \\ 
\bottomrule
\end{tabular}
\end{sc}
\end{small}
\quad
\begin{small}
\begin{sc}
\begin{tabular}{lcccr}
\toprule
Algorithm & $L=2$ & $L=6$ & $L=20$ \\ %
\midrule
GS & $0.0001$ & $0.0001$ & $0.0001$ \\ %
BeS & $0.001$ & $0.0001$ & $0.0001$ \\ %
GS-shrinkage & $0.0001$ & $0.0001$ & $0.0001$ \\ %
BeS-shrinkage & $0.001$ & $0.0001$ & $0.01$ \\ %
Orthogonal ES & $0.0001$ & $0.0001$ & $0.0001$ \\ %
Guided ES & $0.001$ & $0.001$ & $0.01$ \\ %
\bottomrule
\end{tabular}
\end{sc}
\end{small}
\end{center}
\vskip -0.1in
\end{table}

\begin{table}[t]
\caption{$c$ and learning rate for HalfCheetahRandVel.}
\label{tab:cheetahvel_hyper}
\vskip 0.1in
\begin{center}
\begin{small}
\begin{sc}
\begin{tabular}{lcccr}
\toprule
Algorithm & $L=2$ & $L=6$ & $L=20$ \\ 
\midrule
GS & $0.1$ & $0.01$ & $0.01$ \\ 
BeS & $0.01$ & $0.01$ & $0.01$ \\ 
GS-shrinkage & $0.1$ & $0.1$ & $0.01$ \\ 
BeS-shrinkage & $0.1$ & $0.1$ & $0.01$  \\ 
Orthogonal ES & $0.1$ & $0.01$ & $0.01$ \\ 
Guided ES & $0.1$ & $0.1$ & $0.1$ \\ 
\bottomrule
\end{tabular}
\end{sc}
\end{small}
\quad
\begin{small}
\begin{sc}
\begin{tabular}{lcccr}
\toprule
Algorithm & $L=2$ & $L=6$ & $L=20$ \\ %
\midrule
GS & $0.001$ & $0.0001$ & $0.0001$ \\ %
BeS & $0.0001$ & $0.0001$ & $0.0001$ \\ %
GS-shrinkage & $0.001$ & $0.01$ & $0.001$ \\ %
BeS-shrinkage & $0.001$ & $0.01$ & $0.001$ \\ %
Orthogonal ES & $0.001$ & $0.0001$ & $0.0001$ \\ %
Guided ES & $0.001$ & $0.01$ & $0.01$ \\ %
\bottomrule
\end{tabular}
\end{sc}
\end{small}
\end{center}
\vskip -0.1in
\end{table}

\begin{table}[t]
\caption{$c$ and learning rate for ML1-Reach.}
\label{tab:ml1reach_hyper}
\vskip 0.1in
\begin{center}
\begin{small}
\begin{sc}
\begin{tabular}{lcccr}
\toprule
Algorithm & $L=2$ & $L=6$ & $L=20$ \\ 
\midrule
GS & $0.1$ & $0.1$ & $0.01$ \\ 
BeS & $0.1$ & $0.1$ & $0.1$ \\ 
GS-shrinkage & $0.1$ & $0.1$ & $0.1$ \\ 
BeS-shrinkage & $0.1$ & $0.1$ & $0.01$ \\ 
Orthogonal ES & $0.1$ & $0.1$ & $0.1$ \\ 
Guided ES & $0.1$ & $0.1$ & $0.1$ \\ 
\bottomrule
\end{tabular}
\end{sc}
\end{small}
\quad
\begin{small}
\begin{sc}
\begin{tabular}{lcccr}
\toprule
Algorithm & $L=2$ & $L=6$ & $L=20$ \\ %
\midrule
GS & $0.001$ & $0.001$ & $0.0001$ \\ %
BeS & $0.001$ & $0.001$ & $0.01$ \\ %
GS-shrinkage & $0.001$ & $0.001$ & $0.01$ \\ %
BeS-shrinkage & $0.001$ & $0.001$ & $0.0001$ \\ %
Orthogonal ES & $0.001$ & $0.001$ & $0.01$ \\ %
Guided ES & $0.001$ & $0.001$ & $0.001$ \\ %
\bottomrule
\end{tabular}
\end{sc}
\end{small}
\end{center}
\vskip -0.1in
\end{table}

\begin{table}[t]
\caption{$c$ and learning rate for Hopper.}
\label{tab:hopper_hyper}
\vskip 0.1in
\begin{center}
\begin{small}
\begin{sc}
\begin{tabular}{lcccr}
\toprule
Algorithm & $L=2$ & $L=6$ & $L=20$ \\ 
\midrule
GS & $0.1$ & $0.01$ & $0.01$ \\ 
BeS & $0.1$ & $0.01$ & $0.01$ \\ 
GS-shrinkage & $0.1$ & $0.1$ & $0.01$ \\ 
BeS-shrinkage & $0.1$ & $0.1$ & $0.01$ \\ 
Orthogonal ES & $0.1$ & $0.01$ & $0.01$ \\ 
Guided ES & $0.1$ & $0.1$ & $0.1$ \\ 
\bottomrule
\end{tabular}
\end{sc}
\end{small}
\quad
\begin{small}
\begin{sc}
\begin{tabular}{lcccr}
\toprule
Algorithm & $L=2$ & $L=6$ & $L=20$ \\ %
\midrule
GS & $0.001$ & $0.0001$ & $0.0001$ \\ %
BeS & $0.001$ & $0.0001$ & $0.0001$ \\ %
GS-shrinkage & $0.001$ & $0.01$ & $0.0001$ \\ %
BeS-shrinkage & $0.001$ & $0.001$ & $0.0001$ \\ %
Orthogonal ES & $0.001$ & $0.0001$ & $0.0001$ \\ %
Guided ES & $0.01$ & $0.01$ & $0.01$ \\ %
\bottomrule
\end{tabular}
\end{sc}
\end{small}
\end{center}
\vskip -0.1in
\end{table}

\begin{table}[t]
\caption{$c$ and learning rate for ML1-Push.}
\label{tab:ml1push_hyper}
\vskip 0.1in
\begin{center}
\begin{small}
\begin{sc}
\begin{tabular}{lcccr}
\toprule
Algorithm & $L=2$ & $L=6$ & $L=20$ \\ 
\midrule
GS & $0.1$ & $0.1$ & $0.1$ \\ 
BeS & $0.1$ & $0.1$ & $0.1$ \\ 
GS-shrinkage & $0.1$ & $0.1$ & $0.1$ \\ 
BeS-shrinkage & $0.1$ & $0.1$ & $0.1$ \\ 
Orthogonal ES & $0.1$ & $0.1$ & $0.1$ \\ 
Guided ES & $0.1$ & $0.1$ & $0.1$ \\ 
\bottomrule
\end{tabular}
\end{sc}
\end{small}
\quad
\begin{small}
\begin{sc}
\begin{tabular}{lcccr}
\toprule
Algorithm & $L=2$ & $L=6$ & $L=20$ \\ %
\midrule
GS & $0.001$ & $0.001$ & $0.01$ \\ %
BeS & $0.001$ & $0.001$ & $0.01$ \\ %
GS-shrinkage & $0.001$ & $0.001$ & $0.01$ \\ %
BeS-shrinkage & $0.001$ & $0.001$ & $0.01$ \\ %
Orthogonal ES & $0.001$ & $0.001$ & $0.01$ \\ %
Guided ES & $0.001$ & $0.001$ & $0.01$ \\ %
\bottomrule
\end{tabular}
\end{sc}
\end{small}
\end{center}
\vskip -0.1in
\end{table}

\mypara{Additional environments}
We conducted the above experiment on two additional environments, Hopper and ML1-Push.
Hopper is another locomotion environment, similar to Ant and Walker2d, and ML1-Push is a meta-RL manipulation environment similar to ML1-Reach, where the goal is to push an object to some location.
The selected hyperparameters are given in Tables \ref{tab:hopper_hyper} and \ref{tab:ml1push_hyper} and the plots of the test reward against the number of generated trajectories during optimization are given in Figure \ref{fig:rl2}.
When $L=6$, for Hopper Guided ES is the best algorithm, while for ML1-Push BeS, GS-shrinkage, BeS-shrinkage, and Orthogonal ES perform similarly.
For $L=2$ and $L=20$, GS appears to outperform the other algorithms. We suspect that a possible reason is that standardizing the rewards during optimization and searching over a grid of learning rates compensates for errors in the magnitude of the gradient estimate, and thus gives a bigger advantage to GS.

\begin{figure*}[ht]
\vskip 0.1in
    \begin{center}
        \includegraphics[width=\textwidth]{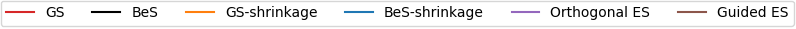}
    \end{center}%
    \begin{subfigure}[b]{0.33\textwidth}
         \includegraphics[width=\textwidth]{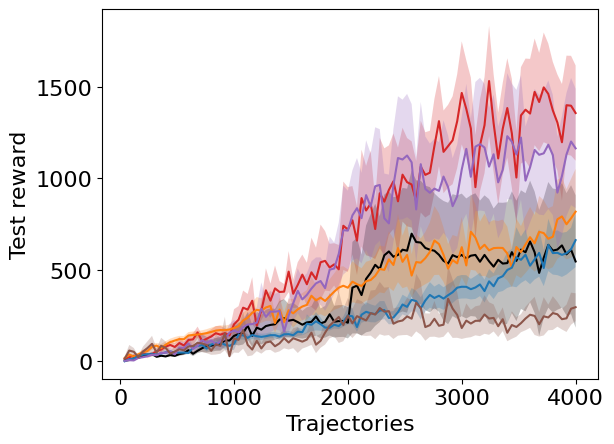}
         \caption{Hopper, $L=2$}
    \end{subfigure}
    \hfill
    \begin{subfigure}[b]{0.33\textwidth}
         \includegraphics[width=\textwidth]{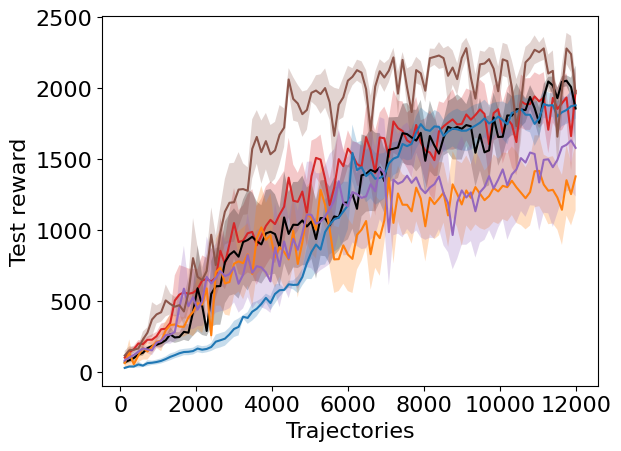}
         \caption{Hopper, $L=6$}
    \end{subfigure}
    \hfill
    \begin{subfigure}[b]{0.33\textwidth}
         \includegraphics[width=\textwidth]{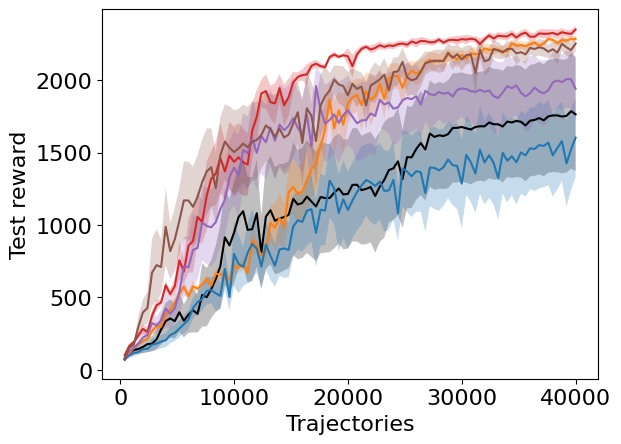}
         \caption{Hopper, $L=20$}
    \end{subfigure}%
    
    \begin{subfigure}[b]{0.33\textwidth}
         \includegraphics[width=\textwidth]{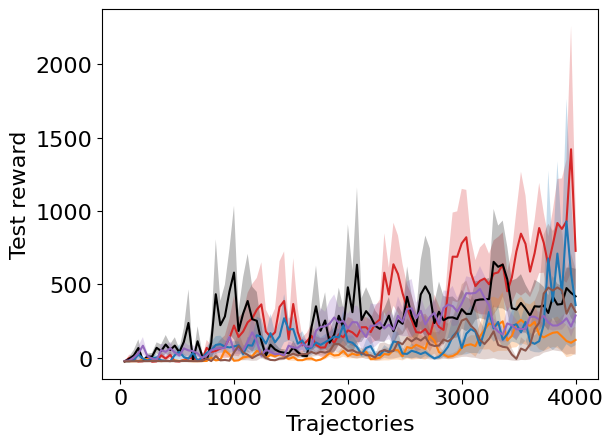}
         \caption{ML1-Push, $L=2$}
    \end{subfigure}
    \hfill
    \begin{subfigure}[b]{0.33\textwidth}
         \includegraphics[width=\textwidth]{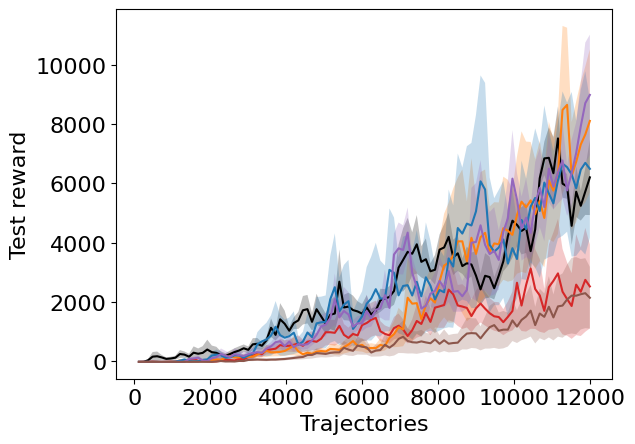}
         \caption{ML1-Push, $L=6$}
    \end{subfigure}
    \hfill
    \begin{subfigure}[b]{0.33\textwidth}
         \includegraphics[width=\textwidth]{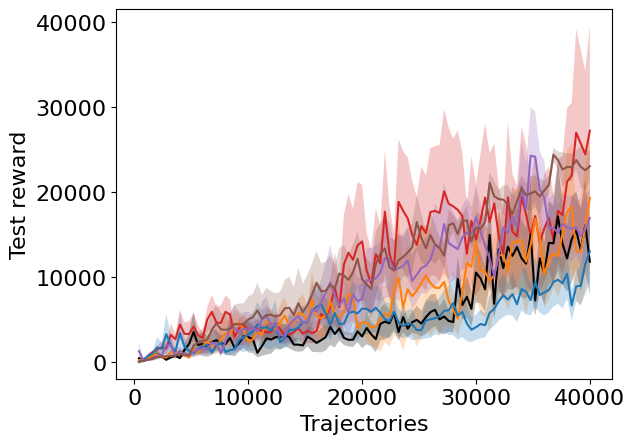}
         \caption{ML1-Push, $L=20$}
    \end{subfigure}%
\caption{RL with various $L$, additional environments}
\label{fig:rl2}
\vskip -0.1in
\end{figure*}

\subsection{Ablations}\label{sec:ablationdetails}

The experimental setup is the same as described for the main online RL experiments, in Appendix \ref{sec:rldetails}.
For brevity, we restrict to the Ant environment and set $L=20$.
The selected hyperparameters are given in the two parts of Table \ref{tab:ablation_hyper}, for i) a MLP policy instead of linear policy and ii) antithetic gradient estimator instead of forward difference gradient estimator.

\begin{table}[t]
\caption{Selected hyperparameters for Ant, $L=20$, with MLP policy (left) or antithetic gradient estimator (right).}
\label{tab:ablation_hyper}
\vskip 0.1in
\begin{center}
\begin{small}
\begin{sc}
\begin{tabular}{lccr}
\toprule
Algorithm & $c$ & Learning rate \\
\midrule
GS & $0.01$ & $0.0001$  \\ 
BeS & $0.01$ & $0.0001$  \\ 
GS-shrinkage & $0.1$ & $0.01$  \\ 
BeS-shrinkage & $0.1$ & $0.01$  \\ 
Orthogonal ES & $0.01$ & $0.0001$  \\ 
Guided ES & $0.1$ & $0.01$  \\  
\bottomrule
\end{tabular}
\end{sc}
\end{small}
\quad
\begin{small}
\begin{sc}
\begin{tabular}{lccr}
\toprule
Algorithm & $c$ & Learning rate \\
\midrule
GS & $0.01$ & $0.0001$  \\ 
BeS & $0.01$ & $0.0001$  \\ 
GS-shrinkage & $0.01$ & $0.0001$  \\ 
BeS-shrinkage & $0.01$ & $0.01$  \\ 
Orthogonal ES & $0.01$ & $0.0001$  \\ 
Guided ES & $0.1$ & $0.01$  \\ 
\bottomrule
\end{tabular}
\end{sc}
\end{small}
\end{center}
\vskip -0.1in
\end{table}

\subsection{Comparing to baselines on DFO benchmarks}\label{sec:benchmarkdets}

\mypara{Optimization and testing}
The parameters are initialized by sampling from $\nN(0,\II)$. 
There are $100$ rounds, each consisting of $10$ optimization iterations and one test step.
In more detail, every optimization iteration has the following steps:
\begin{enumerate}
    \item Sample $L$ directions $\epsilon_l$.
    \item For each direction, obtain a noisy evaluation at the parameters $\theta+c\epsilon_l$ and another at the parameters $\theta$. 
    \item Using those $2L$ numbers, compute the gradient estimate \eqref{eq:fdempirical} with $N=1$.
    \item Take a gradient descent step on $\theta$ with learning rate $\eta$.
\end{enumerate}
and at each test step, compute the objective at the current parameters.

\mypara{Hyperparameter search}
We ran this experiment for $L=\{10, 100\}$ and $N=1$, setting the noise level in Nevergrad to $0.1$.
Hyperparameters are the spacing $c$, chosen from $\{0.01, 0.1\}$, and the SGD learning rate $\eta$, chosen from $\{0.000001, 0.00001, 0.0001, 0.001, 0.01\}$. 
The values chosen are the ones that minimize the objective at the end of the $100$ rounds, averaged over $3$ randomly generated seeds different from those used in Figure \ref{fig:benchmark}.
Tables \ref{tab:sphere_hyper} -- \ref{tab:hm_hyper} show the chosen hyperparameters for each algorithm in the three objectives discussed in the main paper.

\begin{table}[t]
\caption{$c$ and learning rate for \emph{sphere}.}
\label{tab:sphere_hyper}
\vskip 0.1in
\begin{center}
\begin{small}
\begin{sc}
\begin{tabular}{lcccr}
\toprule
Algorithm & $d=10, L=1$ & $d=100, L=10$ & $d=100, L=100$ \cmnt{& $d=1000, L=10$ & $d=1000, L=100$} \\ 
\midrule
GS  & $0.1$ & $0.1$ & $0.1$ \cmnt{& $0.1$ & $0.1$} \\ 
BeS & $0.1$ & $0.1$ & $0.1$ \cmnt{& $0.1$ & $0.1$} \\ 
GS-shrinkage & $0.1$ & $0.1$ & $0.1$ \cmnt{& $0.1$ & $0.1$} \\ 
BeS-shrinkage & $0.1$ & $0.1$ & $0.1$ \cmnt{& $0.1$ & $0.1$} \\ 
Orthogonal ES & $0.1$ & $0.1$ & $0.1$ \cmnt{& $0.1$ & $0.1$} \\ 
Guided ES & $0.1$ & $0.1$ & $0.1$ \cmnt{& $0.1$ & $0.1$} \\ 
\bottomrule
\end{tabular}
\end{sc}
\end{small}
\quad
\begin{small}
\begin{sc}
\begin{tabular}{lcccr}
\toprule
Algorithm & $d=10, L=1$ & $d=100, L=10$ & $d=100, L=100$ \cmnt{& $d=1000, L=10$ & $d=1000, L=100$} \\ %
\midrule
GS & $0.001$ & $0.001$ & $0.001$ \cmnt{& $0.00001$ & $0.0001$} \\ %
BeS & $0.001$ & $0.001$ & $0.001$ \cmnt{& $0.00001$ & $0.0001$} \\ %
GS-shrinkage & $0.01$ & $0.001$ & $0.001$ \cmnt{& $0.00001$ & $0.0001$} \\ %
BeS-shrinkage & $0.01$ & $0.001$ & $0.001$ \cmnt{& $0.000001$ & $0.0001$} \\ %
Orthogonal ES & $0.001$ & $0.001$ & $0.001$ \cmnt{& $0.00001$ & $0.0001$} \\ %
Guided ES & $0.01$ & $0.001$ & $0.01$ \cmnt{& $0.00001$ & $0.0001$} \\ %
\bottomrule
\end{tabular}
\end{sc}
\end{small}
\end{center}
\vskip -0.1in
\end{table}

\begin{table}[t]
\caption{$c$ and learning rate for \emph{rosenbrock}.}
\label{tab:rosenbrock_hyper}
\vskip 0.1in
\begin{center}
\begin{small}
\begin{sc}
\begin{tabular}{lccr}
\toprule
Algorithm & $d=10, L=1$ & $d=100, L=10$ & $d=100, L=100$ \cmnt{& $d=1000, L=10$ & $d=1000, L=100$} \\ 
\midrule
GS & $0.1$ & $0.1$ & $0.1$ \cmnt{& N/A & N/A} \\ 
BeS & $0.1$ & $0.1$ & $0.1$ \cmnt{& N/A & N/A} \\ 
GS-shrinkage & $0.1$ & $0.1$ & $0.1$ \cmnt{& N/A & $0.1$} \\ 
BeS-shrinkage & $0.1$ & $0.1$ & $0.1$ \cmnt{& N/A & $0.1$} \\ 
Orthogonal ES & $0.1$ & $0.1$ & $0.1$ \cmnt{& N/A & N/A} \\ 
Guided ES & $0.1$ & $0.1$ & $0.1$ \cmnt{& N/A & $0.1$} \\ 
\bottomrule
\end{tabular}
\end{sc}
\end{small}
\quad
\begin{small}
\begin{sc}
\begin{tabular}{lccr}
\toprule
Algorithm & $d=10, L=1$ & $d=100, L=10$ & $d=100, L=100$ \cmnt{& $d=1000, L=10$ & $d=1000, L=100$} \\ %
\midrule
GS & $0.000001$ & $0.000001$ & $0.000001$ \cmnt{& N/A & N/A} \\ %
BeS & $0.000001$ & $0.000001$ & $0.000001$ \cmnt{& N/A & N/A} \\ %
GS-shrinkage & $0.000001$ & $0.000001$ & $0.000001$ \cmnt{& N/A & $0.000001$} \\ %
BeS-shrinkage & $0.000001$ & $0.000001$ & $0.000001$ \cmnt{& N/A & $0.000001$} \\ %
Orthogonal ES & $0.000001$ & $0.000001$ & $0.000001$ \cmnt{& N/A & N/A} \\ %
Guided ES & $0.000001$ & $0.000001$ & $0.00001$ \cmnt{& N/A & $0.000001$} \\ %
\bottomrule
\end{tabular}
\end{sc}
\end{small}
\end{center}
\vskip -0.1in
\end{table}

\begin{table}[t]
\caption{$c$ and learning rate for \emph{hm}.}
\label{tab:hm_hyper}
\vskip 0.1in
\begin{center}
\begin{small}
\begin{sc}
\begin{tabular}{lcccr}
\toprule
Algorithm & $d=10, L=1$ & $d=100, L=10$ & $d=100, L=100$ \cmnt{& $d=1000, L=10$ & $d=1000, L=100$} \\ 
\midrule
GS & $0.1$ & $0.1$ & $0.1$ \cmnt{& $0.1$ & $0.1$} \\ 
BeS & $0.1$ & $0.1$ & $0.1$ \cmnt{& $0.1$ & $0.1$} \\ 
GS-shrinkage & $0.1$ & $0.1$ & $0.1$ \cmnt{& $0.1$ & $0.1$} \\  
BeS-shrinkage & $0.1$ & $0.1$ & $0.1$ \cmnt{& $0.1$ & $0.1$} \\ 
Orthogonal ES & $0.1$ & $0.1$ & $0.1$ \cmnt{& $0.1$ & $0.1$} \\ 
Guided ES & $0.1$ & $0.1$ & $0.1$ \cmnt{& $0.1$ & $0.1$} \\ 
\bottomrule
\end{tabular}
\end{sc}
\end{small}
\quad
\begin{small}
\begin{sc}
\begin{tabular}{lcccr}
\toprule
Algorithm  & $d=10, L=1$ & $d=100, L=10$ & $d=100, L=10$ \cmnt{& $d=1000, L=10$ & $d=1000, L=100$} \\ %
\midrule
GS & $0.001$ & $0.0001$ & $0.001$ \cmnt{& $0.000001$ & $0.00001$} \\ %
BeS & $0.001$ & $0.0001$ & $0.001$ \cmnt{& $0.000001$ & $0.00001$} \\ %
GS-shrinkage & $0.001$ & $0.0001$ & $0.001$ \cmnt{& $0.00001$ & $0.00001$} \\ %
BeS-shrinkage & $0.001$ & $0.0001$ & $0.001$ \cmnt{& $0.000001$ & $0.0001$} \\ %
Orthogonal ES & $0.001$ & $0.0001$ & $0.001$ \cmnt{& $0.000001$ & $0.00001$} \\ %
Guided ES & $0.001$ & $0.001$ & $0.001$ \cmnt{& $0.00001$ & $0.00001$} \\ %
\bottomrule
\end{tabular}
\end{sc}
\end{small}
\end{center}
\vskip -0.1in
\end{table}


\end{document}